\PassOptionsToPackage{numbers,sort&compress}{natbib}
\documentclass{article}

\usepackage[preprint]{neurips_2026}
\usepackage[utf8]{inputenc}
\usepackage[T1]{fontenc}
\usepackage{graphicx}
\usepackage{amsmath,amssymb,amsthm,amsfonts,bm}
\usepackage{physics}
\usepackage{booktabs}
\usepackage{microtype}
\usepackage{nicefrac}
\usepackage{hyperref}
\usepackage{url}
\hypersetup{hidelinks}
\usepackage[nameinlink,capitalise,noabbrev]{cleveref}
\usepackage{xcolor}
\usepackage{enumitem}
\usepackage{mathtools}
\usepackage{caption}
\usepackage{subcaption}
\usepackage{array}
\usepackage[ruled,vlined,linesnumbered]{algorithm2e}
\usepackage{tikz}

\usetikzlibrary{arrows.meta,positioning,calc,fit,backgrounds}
\newtheorem{theorem}{Theorem}[section]
\newtheorem{proposition}[theorem]{Proposition}
\newtheorem{lemma}[theorem]{Lemma}
\newtheorem{corollary}[theorem]{Corollary}

\theoremstyle{definition}
\newtheorem{definition}[theorem]{Definition}
\theoremstyle{remark}
\newtheorem{remark}[theorem]{Remark}

\crefname{theorem}{Theorem}{Theorems}
\Crefname{theorem}{Theorem}{Theorems}
\crefname{proposition}{Proposition}{Propositions}
\Crefname{proposition}{Proposition}{Propositions}
\crefname{lemma}{Lemma}{Lemmas}
\Crefname{lemma}{Lemma}{Lemmas}
\crefname{corollary}{Corollary}{Corollaries}
\Crefname{corollary}{Corollary}{Corollaries}
\crefname{definition}{Definition}{Definitions}
\Crefname{definition}{Definition}{Definitions}
\crefname{remark}{Remark}{Remarks}
\Crefname{remark}{Remark}{Remarks}
\crefname{assumption}{Assumption}{Assumptions}
\Crefname{assumption}{Assumption}{Assumptions}
\crefname{equation}{Eq.}{Eqs.}
\Crefname{equation}{Eq.}{Eqs.}
\crefname{table}{Table}{Tables}
\Crefname{table}{Table}{Tables}
\crefname{algocf}{Alg.}{Algs.}
\Crefname{algocf}{Algorithm}{Algorithms}
\crefname{appendix}{Appendix}{Appendices}
\Crefname{appendix}{Appendix}{Appendices}

\newcommand{\pa}{\operatorname{pa}}
\newcommand{\ch}{\operatorname{ch}}
\DeclareMathOperator{\dist}{dist}
\newcommand{\readout}{\operatorname{Readout}}

\newcolumntype{P}{r@{{\scriptsize${}\pm{}$}}>{\scriptsize}l}

\usepackage{amsmath,amsfonts,bm}

\def\eqref#1{equation~\ref{#1}}

\def\1{\bm{1}}

\def\rp{{\textnormal{p}}}

\DeclareMathAlphabet{\mathsfit}{\encodingdefault}{\sfdefault}{m}{sl}
\SetMathAlphabet{\mathsfit}{bold}{\encodingdefault}{\sfdefault}{bx}{n}

\newcommand{\R}{\mathbb{R}}

\title{SpatialPC: Deep Computation with Sublinear Local Credit Propagation}

\author{%
  Benjamin Prada\thanks{Equal contribution.}
  \qquad Shion Matsumoto\footnotemark[1]
  \qquad Ankur Arjun Mali\\\\
  Trustworthy Knowledge-Driven AI Lab\\
  Bellini College of Artificial Intelligence, Cybersecurity and Computing\\
  University of South Florida\\
  Tampa, FL 33620\\
  \texttt{\{prada, matsumoto, ankurarjunmali\}@usf.edu}
}

\begin{document}
\maketitle

\begin{abstract}
Local learning rules are parallel only up to the speed at which information about the loss can travel. On a chain of $N$ states, any rule with bounded communication range needs $\Theta(N)$ rounds before the loss can influence the earliest parameters, and this latency is usually attributed to depth. We show that it is a property of the communication geometry rather than of depth. For $N=m^d$ states placed at constant density in $d$ dimensions, a light-cone argument gives an $\Omega(N^{1/d})$ lower bound on the number of bounded-range rounds that any local algorithm needs before its gradient estimates can be sign-correct at every parameter. We construct a unit-range, degree-$2d$ directed acyclic graph that contains a Hamiltonian computation of length $N-1$ and in which every state lies within directed distance $d(m-1)=\Theta(N^{1/d})$ of the output, matching the lower bound. In the realizable-circuit model, depth and credit depth are therefore independent parameters: represented depth $\Theta(N)$ and causal credit latency $\Theta(N^{1/d})$ coexist, with a separation of order $N^{1-1/d}$. SpatialPC realizes this separation within predictive coding. Credit is propagated by adjoint prediction dynamics (APD), a local update whose $K$-round state on a settled acyclic graph is the truncated reverse-path sum $a^K=\sum_{k<K}(J^\top)^k b$, with error exactly $(J^\top)^K a^\star$. Graph geometry fixes when credit can first arrive at a state; the Jacobian-weighted path tail fixes how many further rounds accurate credit requires. A sparse two-dimensional Moore compiler retains the $N-1$ backbone with $\tfrac32N-2$ unit-range edges and has measured worst-node reachability $\sqrt N+2$ up to $N=65{,}536$, where a 65{,}535-edge backbone is reached in 258 local hops. With $T=2\sqrt N$ rounds, SpatialPC keeps finite-round task gradients aligned with settled gradients (cosine $0.93$ on MNIST and $0.97$ on Fashion-MNIST at $N=256$) and reaches $97.6\%$ and $88.1\%$ test accuracy after 20 epochs.
\end{abstract}

\section{Introduction}
\label{sec:introduction}

Backpropagation assigns credit with a global reverse pass \cite{Rumelhart1986}: the error at the output is carried back through every layer in sequence, using the transposed forward weights and a separate backward phase. On a general-purpose processor this is efficient. On a system whose computation is distributed over physically local units, whether cortical circuits \cite{Lillicrap2020} or neuromorphic and in-memory hardware, it is not. Such a system is an embodied learner in the sense of embodied cognition \cite{varela1991embodied,pfeifer2006body}: its computation occupies physical space, so the time for any signal, including a learning signal, to cross it is set by the geometry of the substrate rather than by an operation count, and a learning rule for such a system must respect that constraint. This has motivated local learning rules whose updates depend only on quantities available at a unit and its neighbors. Predictive coding (PC) \cite{millidge2022predictivecodingtheoreticalexperimental,salvatori2025survey}, equilibrium propagation \cite{scellier2017equilibrium}, forward-forward \cite{hinton2022forwardforwardalgorithmpreliminaryinvestigations}, local representation alignment or target propagation \cite{lee2015difference, ororbia2019biologically}, and random-feedback rules \cite{Lillicrap2020} are the main families. PC is the most developed among them: its dynamics are defined on arbitrary computation graphs \cite{millidge2020predictive,salvatori2022graphs}, they recover the backpropagation gradient at equilibrium \cite{salvatori2022reverse}, and they have been trained at the scale of 1000-layer residual networks \cite{seely2026pcalm}, convolutional networks on CIFAR and TinyImageNet \cite{ha2026metapcn}, and a VGG network on ImageNet \cite{kerjan2026imagenet}.

The obstacle that recurs in this work is depth. On an $L$-layer chain the loss is generated at the last layer and the first layer is $L-1$ local interactions away. In a synchronous model, at least $L-1$ communication rounds pass before the earliest parameters can receive any information about the loss; in continuous-time inference the same distance appears as slow, attenuated, or unstable error propagation \cite{goemaere2026epc,innocenti2025mupc,qi2026deeper}. The remedies proposed so far act on the dynamics: depth-aware parameterization \cite{innocenti2025mupc,innocenti2026limits}, corrected error balance and timing \cite{qi2026deeper}, direct error pathways to selected layers \cite{mohammadi2026hep}, primal--dual updates \cite{seely2026pcalm}, incremental schedules \cite{salvatori2024ipc}, and error-based formulations without signal decay \cite{goemaere2026epc}. All of them take the chain as given, and so all of them inherit its $\Theta(L)$ communication distance.

The chain conflates two quantities. The \emph{represented depth} of a computation graph is the length of its longest directed path, the number of sequential operations it can express. The \emph{reachability depth} is the largest number of edges separating any state from the output, the number of bounded-range rounds before every state can have heard from the loss. On a chain both equal $L-1$. The communication graph of a physical substrate, however, is not a chain: a two-dimensional fabric of $N$ units with bounded range and bounded degree has diameter of order $\sqrt N$, and a three-dimensional one of order $N^{1/3}$. Whether a long computation can be laid out on such a fabric so that its reachability depth matches the diameter while its represented depth stays linear has not been asked, and no bound has been established on the latency that locality itself permits.

This paper answers both questions. We fix a locality model in which $N=m^d$ units sit at constant density in $d$ dimensions, communicate within bounded range, and receive the target only at the output. A light-cone argument shows that no existing local learning algorithm, irrespective of how it computes, can deliver sign-correct gradient estimates to every parameter in fewer than $\Omega(N^{1/d})$ rounds after the target is revealed (\cref{thm:light-cone}). A snake-ordered orientation (Boustrophedon) of the unit grid contains a Hamiltonian computation of length $N-1$, uses only unit-length edges of degree at most $2d$, and places every state within directed distance $d(m-1)$ of the output (\cref{thm:spatial-compiler}). Together,
\begin{equation}
    D_{\mathrm{seq}}=N-1=\Theta(N),
    \qquad
    D_{\mathrm{reach}}=d\bigl(N^{1/d}-1\bigr)=\Theta\bigl(N^{1/d}\bigr),
    \label{eq:intro-scaling-law}
\end{equation}
so the ratio $D_{\mathrm{seq}}/D_{\mathrm{reach}}=\Theta(N^{1-1/d})$ is unbounded, and the exponent $1/d$ is the best that any local algorithm can achieve. In the language of realizable circuit complexity \cite{prada2025realizable}, circuits embedded in $d$-dimensional space under causality and bounded fan-in, depth and credit depth (reachability depth) are independent parameters (\cref{cor:optimal-exponent}).

\definecolor{spblue}{RGB}{38,86,145}
\definecolor{spred}{RGB}{168,55,55}
\definecolor{spgray}{RGB}{105,105,105}
\tikzset{
  >={Latex[length=1.7mm,width=1.1mm]},
  nodept/.style={circle,draw=black!70,fill=white,minimum size=3.7mm,inner sep=0pt,line width=.5pt},
  localedge/.style={draw=black!22,line width=.42pt},
  backbone/.style={->,draw=black!76,line width=.85pt},
  credit/.style={->,draw=spblue,line width=1.15pt},
  smalltxt/.style={font=\scriptsize,text=black!82,align=center},
}

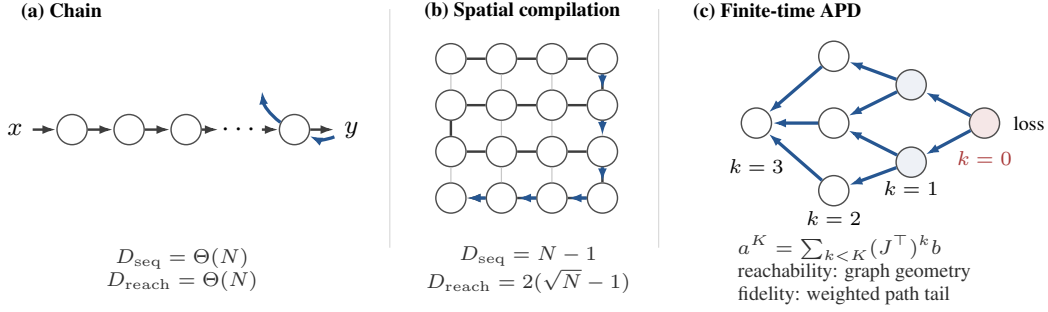
\begin{figure}[t]
\centering
\resizebox{\linewidth}{!}{%
\begin{tikzpicture}[x=1cm,y=1cm]
\node[font=\scriptsize\bfseries,anchor=west] at (0.0,4.15) {(a) Chain};
\node[font=\small] (xin) at (0.05,2.70) {$x$};
\node[nodept] (c1) at (0.75,2.70) {};
\node[nodept] (c2) at (1.45,2.70) {};
\node[nodept] (c3) at (2.15,2.70) {};
\node at (2.82,2.70) {$\cdots$};
\node[nodept] (cN) at (3.48,2.70) {};
\node[font=\small] (yout) at (4.18,2.70) {$y$};
\draw[backbone] (xin)--(c1);
\draw[backbone] (c1)--(c2);
\draw[backbone] (c2)--(c3);
\draw[backbone] (c3)--(2.55,2.70);
\draw[backbone] (3.08,2.70)--(cN);
\draw[backbone] (cN)--(yout);
\draw[credit,bend left=23] (yout) to (cN);
\draw[credit,bend left=23] (cN) to (3.08,3.18);
\node[smalltxt] at (2.10,0.98) {$D_{\rm seq}=\Theta(N)$\\$D_{\rm reach}=\Theta(N)$};

\draw[black!15] (4.65,1.25)--(4.65,4.18);

\node[font=\scriptsize\bfseries,anchor=west] at (4.95,4.15) {(b) Spatial compilation};
\foreach \r in {0,...,3}{
  \foreach \c in {0,...,3}{
    \pgfmathsetmacro{\xx}{5.40+0.62*\c}
    \pgfmathsetmacro{\yy}{3.57-0.57*\r}
    \node[nodept] (g\r\c) at (\xx,\yy) {};
  }
}
\foreach \r in {0,...,3}{
  \foreach \c in {0,...,2}{
    \pgfmathtruncatemacro{\cp}{\c+1}
    \draw[localedge] (g\r\c)--(g\r\cp);
  }
}
\foreach \r in {0,...,2}{
  \pgfmathtruncatemacro{\rp}{\r+1}
  \foreach \c in {0,...,3}{
    \draw[localedge] (g\r\c)--(g\rp\c);
  }
}
\draw[backbone]
  (g00)--(g01)--(g02)--(g03)--(g13)--(g12)--(g11)--(g10)--
  (g20)--(g21)--(g22)--(g23)--(g33)--(g32)--(g31)--(g30);
\draw[credit] (g03)--(g13);
\draw[credit] (g13)--(g23);
\draw[credit] (g23)--(g33);
\draw[credit] (g33)--(g32);
\draw[credit] (g32)--(g31);
\draw[credit] (g31)--(g30);
\node[smalltxt] at (6.35,0.98) {$D_{\rm seq}=N-1$\\$D_{\rm reach}=2(\sqrt N-1)$};

\draw[black!15] (7.95,1.25)--(7.95,4.18);

\node[font=\scriptsize\bfseries,anchor=west] at (8.25,4.15) {(c) Finite-time APD};
\node[nodept,fill=spred!12] (o) at (11.95,2.78) {};
\node[font=\scriptsize,anchor=west] at (12.18,2.78) {loss};
\node[nodept,fill=spblue!8] (p1a) at (11.05,2.30) {};
\node[nodept,fill=spblue!8] (p1b) at (11.05,3.25) {};
\node[nodept] (p2a) at (10.10,1.95) {};
\node[nodept] (p2b) at (10.10,2.78) {};
\node[nodept] (p2c) at (10.10,3.60) {};
\node[nodept] (p3) at (9.15,2.78) {};
\draw[credit] (o)--(p1a);
\draw[credit] (o)--(p1b);
\draw[credit] (p1a)--(p2a);
\draw[credit] (p1a)--(p2b);
\draw[credit] (p1b)--(p2b);
\draw[credit] (p1b)--(p2c);
\draw[credit] (p2a)--(p3);
\draw[credit] (p2b)--(p3);
\draw[credit] (p2c)--(p3);
\node[font=\scriptsize,text=spred] at (11.95,2.35) {$k=0$};
\node[font=\scriptsize] at (11.05,1.98) {$k=1$};
\node[font=\scriptsize] at (10.10,1.62) {$k=2$};
\node[font=\scriptsize] at (9.15,2.25) {$k=3$};
\node[smalltxt,align=left,anchor=west] at (8.80,0.98)
  {$a^K=\sum_{k<K}(J^\top)^k b$\\reachability: graph geometry\\fidelity: weighted path tail};
\end{tikzpicture}%
}
\caption{\textbf{Represented depth and local credit latency are different resources.} (a) On a chain, represented depth and loss reachability are both $\Theta(N)$. (b) The snake-ordered grid of \cref{thm:spatial-compiler} ($d=2$) keeps an $N-1$ Hamiltonian computation (dark arrows) while every state lies within $2(\sqrt N-1)$ unit-range hops of the output; the blue route is a shortest credit path. (c) APD accumulates reverse paths by length $k$. Geometry fixes the first round at which a state can receive credit; the Jacobian-weighted tail $(J^\top)^Ka^\star$ fixes how many further rounds accurate credit requires.}
\label{fig:spatial_compilation}
\end{figure}

Reachability says when credit can arrive, not whether it is accurate. SpatialPC, the predictive-coding model that realizes the separation, compiles a computation onto a bounded-degree spatial fabric and propagates credit with adjoint prediction dynamics (APD), a local update in which each state predicts its own adjoint from the adjoints of its children and cancels the prediction error. On a settled graph the $K$-round APD state is the sum of all reverse paths of length less than $K$, and its error is exactly the settled adjoint pushed back through $K$ more edges, $a^\star-a^K=(J^\top)^Ka^\star$. The first term to reach a state is fixed by graph distance; the rounds needed after that are fixed by the decay of the Jacobian powers, which depends on the learned weights and the task. We call the latter the \emph{effective credit radius} $R_\epsilon(J,b)$; in the experiments, the feasibility weight and the round budget act on $R_\epsilon$ while $D_{\mathrm{reach}}$ stays fixed.

The implementation uses a sparse two-dimensional compiler built on a Moore traversal \cite{Moore1900}. It keeps the $N-1$ backbone with $\tfrac32N-2$ unit-range edges and has measured worst-node output distance $\sqrt N+2$ for every power of two from $N=16$ to $N=65{,}536$. With $T=2\sqrt N$ synchronous rounds on MNIST and Fashion-MNIST, finite-round task gradients at $N=256$ have cosine $0.93$ and $0.97$ with the settled gradients of the same network, and twenty epochs of training reach $97.64\%$ and $88.10\%$ test accuracy, matching a depth-128 backpropagation reference.

\paragraph{Contributions.}
\begin{enumerate}[leftmargin=*,itemsep=1pt,topsep=2pt]
    \item A tight $\Theta(N^{1/d})$ characterization of worst-case causal credit latency for realizable circuits in $d$ dimensions: a light-cone lower bound that holds for every local learning algorithm (\cref{thm:light-cone}), a unit-range compiler that attains it with depth $N-1$ (\cref{thm:spatial-compiler}), and the consequence that depth and credit depth are independent circuit parameters (\cref{cor:optimal-exponent}).
    \item An exact, edge-local gradient for the finite-round SpatialPC objective (\cref{prop:finite-spatialpc-apd}), together with a path-length decomposition of settled credit whose $K$-round error is $(J^\top)^Ka^\star$ (\cref{prop:finite-neumann}); this separates reachability from fidelity, controls the angle of each local update (\cref{cor:gradient-angle}), and has a forward counterpart in which the feasibility residual certifies the distance to the settled state (\cref{prop:forward-reverse}).
    \item A sparse compiler with a proved edge count of $\tfrac32N-2$ for any Hamiltonian cycle of the grid (\cref{prop:moore-edges}) and measured reachability $\sqrt N+2$ up to $N=65{,}536$.
    \item Experiments showing that settled-gradient alignment is maintained under $T=2\sqrt N$ as the backbone grows from 15 to 255 edges, that the feasibility weight raises gradient cosine from $0.35$ to $0.93$ at constant accuracy, and that fidelity keeps improving for round budgets well beyond $D_{\mathrm{reach}}$.
\end{enumerate}

\Cref{sec:related} places the work among local learning rules, realizable circuit complexity, parallel differentiation, and spatial layout. \Cref{sec:background} provides notation: predictive coding as a constrained problem, the adjoint equation and APD, and the two depths. \Cref{sec:spatialpc} defines realizable circuits, proves the separation and its optimality, and describes the Moore compiler. \Cref{sec:dynamics} specifies the finite-round forward dynamics of SpatialPC and its exact local gradient, and \cref{sec:finite-credit} analyzes the accuracy of finite-round credit. \Cref{sec:experiments} reports structural measurements to $N=65{,}536$ and task experiments to $N=256$. \Cref{sec:discussion} discusses the scope of the bounds, their reading in circuit-complexity terms, physical implementation, and the open problems they raise, and \cref{sec:conclusion} concludes. Proofs omitted from the main text, the numerical verification of the implementation, secondary experiments, and pseudocode are in the appendix.

\section{Related Work}
\label{sec:related}

\medskip\noindent\textbf{Local learning rules.}
Rules that avoid the global reverse pass differ in what they transmit but share the constraint that each update uses information available at a unit and its neighbors \cite{ororbia2024review}. Equilibrium propagation obtains gradients from the response of an energy-based network to a nudged output \cite{scellier2017equilibrium}; forward-forward trains each layer on a local goodness objective with no backward signal \cite{hinton2022forwardforwardalgorithmpreliminaryinvestigations}; random-feedback rules replace transposed weights by fixed random matrices but still propagate layer by layer \cite{Lillicrap2020}. In all of these the output signal crosses the layers one at a time, so the number of local rounds before the first layer can be influenced grows with depth. Synthetic gradients \cite{jaderberg2017dni} remove the dependence by predicting the gradient locally, at the cost of a learned approximation and of locality itself. The locality model of \cref{sec:locality-model} covers all of these rules, and \cref{thm:light-cone} applies to each of them.

\medskip\noindent\textbf{Predictive coding on graphs and at scale.}
Predictive-coding dynamics were defined on chains and extended to arbitrary graph topologies \cite{millidge2020predictive,salvatori2022graphs}, including bidirectional and recurrent architectures for vision \cite{pmlr-v80-wen18a}; at equilibrium they reproduce reverse-mode differentiation on the computation graph \cite{salvatori2022reverse}. Work on scaling has acted on the dynamics: incremental schedules \cite{salvatori2024ipc}, depth-aware parameterization \cite{innocenti2025mupc,innocenti2026limits}, corrected error balance and timing \cite{qi2026deeper}, direct output-error pathways \cite{mohammadi2026hep}, augmented-Lagrangian primal--dual updates \cite{seely2026pcalm}, error-based formulations \cite{goemaere2026epc}, stabilized convolutional models \cite{ha2026metapcn}, and equilibrium-propagation training on ImageNet \cite{kerjan2026imagenet}. These methods change how credit moves over a given graph. The present work changes the graph, and uses the same local stationarity condition as augmented-Lagrangian PC (\cref{app:apd-pcalm}), so it can be combined with any of them.

\medskip\noindent\textbf{Realizable circuit complexity.}
Prada and Mali \cite{prada2025realizable} define the classes $\mathbf{RC}_d(t)$ of circuits realizable in $d$-dimensional space under causality, bounded information flux, finite gate volume with constant fan-in, and local uniformity: a circuit that runs for time $t$ lies within a causal ball of radius $O(t)$, so its size is $O(t^d)$ and its width $O(t^{d-1})$. The classes form a strict hierarchy in $d$ with $\mathbf{RC}_\infty(\mathrm{polylog})=\mathbf{NC}$, and the volume constraint alone forces time $\Omega(n^{1/d})$ for an $n$-gate circuit. The classical antecedents are the bounded-wire-length models of VLSI complexity theory \cite{thompson1980vlsi} and Vit\'anyi's result that low-diameter networks cannot be realized with short wires \cite{vitanyi1988locality}. \Cref{def:realizable} is the single-circuit form of this model, keeping the causality and gate-quantization constraints and singling out depth and credit depth as the parameters of interest, and the geometric budget $T=2\sqrt N$ of \cref{sec:experiments} is the regime $N=O(T^2)$ in which an $N$-gate planar circuit saturates the $\mathbf{RC}_2$ volume bound. The proof of \cref{thm:light-cone} combines the causal-ball argument with an indistinguishability argument of the kind used in the LOCAL model of distributed computing \cite{linial1992locality,peleg2000distributed}; what is specific to learning is the quantity bounded, the sign of a gradient estimate, and the matching upper bound, which shows that the causal radius can be attained by a circuit that still contains a Hamiltonian computation.

\medskip\noindent\textbf{Parallel differentiation and implicit models.}
Reverse-mode differentiation can be organized as triangular solves, parallel scans, or parallel recurrences \cite{naumov2017parallel,wang2020bppsa,song2021accelerating}, which reduce span when long-range reductions or operator composition are available; a bounded-range substrate provides neither. Deep equilibrium models define their states as fixed points and differentiate through the fixed point implicitly \cite{bai2019deq}; SpatialPC also targets a fixed point of local dynamics but differentiates through a finite unrolling and penalizes the terminal residual (\cref{sec:relation-to-pc}).

\medskip\noindent\textbf{Space-filling curves.}
The Hilbert curve and Moore's closed variant of it \cite{Moore1900} are used for locality-preserving traversals in spatial indexing \cite{google-s2}; the Moore curve is generated here by an L-system \cite{Prusinkiewicz1996}. The orders of \cref{sec:upper-bound,sec:moore} serve a different purpose: not to keep consecutive items close, which any Hamiltonian path of the grid does, but to make the oriented unit-grid shortcuts give every point a short directed route to the output.

\section{Preliminaries}
\label{sec:background}

We write predictive coding as a constrained problem, derive the adjoint equation that APD solves, and define the two depths that the rest of the paper separates.

\subsection{Predictive coding as constrained optimization}
Let $G=(V,E)$ be a directed acyclic computation graph with a state $h_v\in\mathbb R^H$ at each node $v$, and write $\pa(v)$ and $\ch(v)$ for the parents and children of $v$. Every non-input state is predicted from its parents, and the prediction error is the local residual
\begin{equation}
    h_v=f_v\bigl(h_{\pa(v)};\theta_v\bigr),
    \qquad
    r_v(h,\theta)=h_v-f_v\bigl(h_{\pa(v)};\theta_v\bigr).
    \label{eq:dag-constraint}
\end{equation}
Stacking the node-wise predictions into a map $F_\theta(h;x)$ that also injects the input $x$ at the input node gives the global residual
\begin{equation}
    r(h,\theta)=h-F_\theta(h;x).
    \label{eq:global-residual}
\end{equation}
A feedforward evaluation of the graph is a state with $r(h,\theta)=0$, so supervised learning on $G$ is the constrained problem
\begin{equation}
    \min_{h,\theta}\ \ell(h_O,y)
    \qquad\text{s.t.}\qquad r(h,\theta)=0,
    \label{eq:constrained-pc}
\end{equation}
where $O\subseteq V$ is the set of output nodes and $\ell$ is the task loss. Predictive coding relaxes the constraints to a quadratic penalty,
\begin{equation}
    \mathcal F_{\mathrm{PC}}(h,\theta)
    =\ell(h_O,y)+\frac{\beta}{2}\|r(h,\theta)\|_2^2,
    \label{eq:pc-loss}
\end{equation}
and alternates local state inference with local parameter updates \cite{millidge2020predictive,salvatori2022graphs}.

\subsection{Adjoint prediction dynamics}
\label{sec:apd}
Attach a multiplier $a_v$ to each local constraint in \cref{eq:constrained-pc} and form the Lagrangian
\begin{equation}
    \mathcal L(h,a,\theta)
    =\ell(h_O,y)-a^\top\bigl(h-F_\theta(h;x)\bigr).
    \label{eq:pc-lagrangian}
\end{equation}
Let $J=\partial F_\theta(h;x)/\partial h$ be the state Jacobian and $b=\nabla_h\ell(h_O,y)$ the loss forcing, which is supported on $O$. Stationarity of $\mathcal L$ in $h$ is the adjoint equation
\begin{equation}
    a=b+J^\top a,
    \label{eq:settled-adjoint}
\end{equation}
which at node $v$ involves only its children:
\begin{equation}
    a_v=b_v+\sum_{w\in\ch(v)}J_{w\leftarrow v}^{\top}a_w,
    \qquad
    J_{w\leftarrow v}=\frac{\partial f_w}{\partial h_v}.
    \label{eq:local-adjoint}
\end{equation}
Adjoint prediction dynamics (APD) treats the right-hand side as a local prediction of the adjoint and cancels the prediction error,
\begin{equation}
    \widehat a_v
    =b_v+\sum_{w\in\ch(v)}J_{w\leftarrow v}^{\top}a_w,
    \qquad
    a_v\leftarrow a_v-(a_v-\widehat a_v)=\widehat a_v.
    \label{eq:apd-update}
\end{equation}
Applied synchronously to all nodes, \cref{eq:apd-update} is the Jacobi iteration $a^{k+1}=b+J^\top a^k$ for \cref{eq:settled-adjoint}. The parameter gradient is assembled from the same local quantities,
\begin{equation}
    \nabla_\theta\ell
    =\left(\frac{\partial F_\theta}{\partial\theta}\right)^{\top}a,
    \label{eq:parameter-gradient-general}
\end{equation}
so that for a linear edge $f_v\ni W_{uv}z_u$ the block $\nabla_{W_{uv}}\ell$ is the outer product of the postsynaptic adjoint $a_v$ and the presynaptic activity $z_u$, scaled by the edge gain. \Cref{eq:settled-adjoint} is also the state-stationarity condition of augmented-Lagrangian predictive coding \cite{seely2026pcalm}, with $a$ equal to minus the augmented multiplier (\cref{app:apd-pcalm}).

\subsection{Represented depth and reachability depth}
\label{sec:two-depths}
For a DAG $G$ let $|p|$ be the number of edges of a directed path $p$. The \emph{represented depth} is the length of the longest directed computation, and for an output set $O$ the \emph{reachability depth} is the largest shortest-path distance from a node to the output:
\begin{equation}
    D_{\mathrm{seq}}(G):=\max_{p\ \text{directed}} |p|,
    \qquad
    D_{\mathrm{reach}}(G,O):=\max_{v\in V}\ \min_{p:\,v\rightsquigarrow O}|p|.
    \label{eq:dseq}
\end{equation}
$D_{\mathrm{reach}}$ is a communication quantity: a message that originates at $O$ and moves one edge per round cannot have reached every node before $D_{\mathrm{reach}}$ rounds have elapsed. A chain on $N$ nodes has $D_{\mathrm{seq}}=D_{\mathrm{reach}}=N-1$. The next section shows that on a $d$-dimensional substrate the two can differ by a factor of order $N^{1-1/d}$, and that no local algorithm can do better.

\section{Spatial Compilation: Linear Depth with Sublinear Reachability}
\label{sec:spatialpc}

On a chain the two depths of \cref{sec:two-depths} coincide. We now place the computation on a $d$-dimensional substrate and show that they need not. \Cref{sec:locality-model} defines realizable circuits and the locality model; \cref{sec:upper-bound} constructs a realizable circuit with $D_{\mathrm{seq}}=N-1$ and $D_{\mathrm{reach}}=d(N^{1/d}-1)$; \cref{sec:lower-bound} proves that credit depth $\Omega(N^{1/d})$ is unavoidable for realizable circuits and that $\Omega(N^{1/d})$ rounds are necessary for every local algorithm; and \cref{sec:moore} describes the sparser compiler used in the experiments.

\subsection{Realizable circuits and the compilation problem}
\label{sec:locality-model}

\begin{definition}[Substrate]
\label{def:substrate}
A \emph{substrate} is a finite set $V$ of $N$ sites together with an embedding $\phi:V\to\mathbb R^d$; distances are $\dist(u,v):=\|\phi(u)-\phi(v)\|_1$. The substrate has \emph{density at most $C_d$} if every $\ell_1$-ball of radius $\varrho\ge1$ contains at most $C_d\varrho^d$ sites. The unit grid $\{0,\dots,m-1\}^d$ with one site per lattice point has density at most $3^d$.
\end{definition}

\begin{definition}[Realizable circuit]
\label{def:realizable}
A \emph{realizable circuit} of range $\rho$ and degree $\Delta$ on a substrate is a DAG $G=(V,E)$ on the sites with $\dist(u,v)\le\rho$ for every edge and $\deg(v)\le\Delta$ for every vertex. Its \emph{depth} is $D_{\mathrm{seq}}(G)$ and, for an output set $O$ reachable from every vertex, its \emph{credit depth} is the reachability depth $D_{\mathrm{reach}}(G,O)$ of \cref{sec:two-depths}. A family $\{G_N\}$ is \emph{realizable in dimension $d$} if its substrates have density at most a common constant $C_d$ and $\rho$, $\Delta$, and $|O|$ are bounded in $N$.
\end{definition}
This is the single-circuit form of the realizable circuit model of Prada and Mali \cite{prada2025realizable}, restricted to its causality and gate-quantization constraints (bounded wire length, bounded degree, constant density), with the depth of the circuit and the reverse eccentricity of its outputs singled out as the two parameters of interest; the bounded-wire-length models of VLSI theory \cite{thompson1980vlsi,vitanyi1988locality} are its classical antecedents. The compilation problem fixes the depth and asks for the credit depth.

\begin{definition}[Admissible graphs]
\label{def:admissible}
Let $P_N=(v_0,\dots,v_{N-1})$ be a directed computational backbone with topological order $\prec$, and let $\rho$ be a communication range and $\Delta$ a degree bound. The admissible graphs are the realizable circuits that contain the backbone and respect its order,
\begin{equation}
\begin{aligned}
\mathcal G(P_N,\phi;\rho,\Delta)=\bigl\{G=(V,E):\;&P_N\subseteq E,\ (u,v)\in E\Rightarrow u\prec v,\\
&\dist(u,v)\le\rho,\ \deg(v)\le\Delta\ \text{for all } v\bigr\}.
\end{aligned}
\label{eq:admissible-graphs}
\end{equation}
\end{definition}
Every admissible graph has depth $D_{\mathrm{seq}}(G)=N-1$, the largest possible on $N$ vertices. The \emph{spatial compilation problem} is to choose the embedding and the shortcut edges so as to minimize the credit depth,
\begin{equation}
    \min_{\phi}\ \min_{G\in\mathcal G(P_N,\phi;\rho,\Delta)}
    D_{\mathrm{reach}}(G,O),
    \label{eq:geometric-compile}
\end{equation}
subject to a prescribed output set $O$. The objective is the number of bounded-range rounds after which every state can have received information generated at the output.

\begin{definition}[Local learning algorithm]
\label{def:local-algorithm}
A learning algorithm on a substrate is \emph{$\rho$-local} if it proceeds in synchronous rounds in which the memory of each site is updated as a function of the memories of the sites within distance $\rho$ of it, itself included. The input $x$ and the parameters reside at their sites throughout, and the target $y$ is written into the memories of the output sites $O$ at round $0$ and nowhere else.
\end{definition}
Forward computation, credit propagation, and any auxiliary signaling are charged to the same rounds, and nothing is assumed about what the sites compute; the only restriction is on where their inputs come from. Both the compiler below and APD in \cref{sec:dynamics} are $\rho$-local with $\rho=1$.

\subsection{Upper bound: the snake-ordered grid}
\label{sec:upper-bound}

For $m\ge2$ and $d\ge1$ let $\pi_{m,d}$ be the \emph{reflected snake order} on $\{0,\dots,m-1\}^d$: for $d=1$, $\pi_{m,1}(i)=i$; for $d>1$, writing a point as $(i,v')$ with $i$ its first coordinate,
\begin{equation}
    \pi_{m,d}(i,v')=i\,m^{d-1}+
    \begin{cases}
        \pi_{m,d-1}(v'), & i\ \text{even},\\
        m^{d-1}-1-\pi_{m,d-1}(v'), & i\ \text{odd}.
    \end{cases}
    \label{eq:snake-order}
\end{equation}
The order visits the slabs $i=0,\dots,m-1$ in turn, traversing the $(d-1)$-dimensional snake forward in even slabs and backward in odd ones. Consecutive points are unit neighbors. Let $v_{\mathrm{in}}$ and $v_{\mathrm{out}}$ be the first and last points of the order; both are corners of the grid.

\begin{lemma}[Monotone edges]
\label{lem:monotone}
In the order $\pi=\pi_{m,d}$, every $v\neq v_{\mathrm{out}}$ has a unit neighbor $u$ with $\pi(u)>\pi(v)$ and $\|u-v_{\mathrm{out}}\|_1=\|v-v_{\mathrm{out}}\|_1-1$, and every $v\ne v_{\mathrm{in}}$ has a unit neighbor $u$ with $\pi(u)<\pi(v)$ and $\|u-v_{\mathrm{in}}\|_1=\|v-v_{\mathrm{in}}\|_1-1$.
\end{lemma}
The two statements must be carried together through the induction on $d$, because an odd slab is traversed backward and its last point is the first point of the lower-dimensional snake; the proof is in \cref{app:proofs-compilation}.

\begin{theorem}[Snake-ordered grid compiler]
\label{thm:spatial-compiler}
Let $d\ge1$, $m\ge2$, $N=m^d$, and let $G_{m,d}$ be the DAG on $\{0,\dots,m-1\}^d$ obtained by orienting every unit-grid adjacency from its earlier to its later endpoint in $\pi_{m,d}$, with input $v_{\mathrm{in}}$ and output $O=\{v_{\mathrm{out}}\}$. Then:
\begin{enumerate}[label=(\roman*),leftmargin=*,itemsep=0pt,topsep=1pt]
    \item $D_{\mathrm{seq}}(G_{m,d})=N-1$;
    \item for every $v$, the directed distance from $v$ to $v_{\mathrm{out}}$ equals $\|v-v_{\mathrm{out}}\|_1$, and the directed distance from $v_{\mathrm{in}}$ to $v$ equals $\|v-v_{\mathrm{in}}\|_1$; consequently
    \begin{equation}
        D_{\mathrm{reach}}(G_{m,d},O)=d(m-1)=d\bigl(N^{1/d}-1\bigr),
        \label{eq:snake-reach}
    \end{equation}
    and the forward eccentricity of $v_{\mathrm{in}}$ is the same number;
    \item $|E|=d(m-1)m^{d-1}$, every edge has unit length, and every vertex has degree at most $2d$.
\end{enumerate}
\end{theorem}

\begin{proof}
(i) Consecutive points of $\pi_{m,d}$ are unit neighbors, so the order is a Hamiltonian path of length $N-1$ in $G_{m,d}$; no directed path in a DAG on $N$ vertices is longer.

(ii) Every edge changes the $\ell_1$ distance to $v_{\mathrm{out}}$ by exactly one, so any directed path from $v$ to $v_{\mathrm{out}}$ has at least $\|v-v_{\mathrm{out}}\|_1$ edges. By \cref{lem:monotone}, from every $v\neq v_{\mathrm{out}}$ there is an out-edge of $G_{m,d}$ whose head is one unit closer to $v_{\mathrm{out}}$; following such edges reaches $v_{\mathrm{out}}$ in exactly $\|v-v_{\mathrm{out}}\|_1$ steps. The maximum of $\|v-v_{\mathrm{out}}\|_1$ over the grid is attained at the corner opposite $v_{\mathrm{out}}$ and equals $d(m-1)$. The forward statement follows from the second half of \cref{lem:monotone} by the same argument with the roles of $v_{\mathrm{in}}$ and $v_{\mathrm{out}}$ exchanged.

(iii) The grid has $m^{d-1}(m-1)$ unit adjacencies along each of the $d$ axes, each vertex has at most $2d$ unit neighbors, and every edge joins unit neighbors.
\end{proof}

The construction therefore separates the two depths by
\begin{equation}
    \frac{D_{\mathrm{seq}}}{D_{\mathrm{reach}}}
    =\frac{N-1}{d\,(N^{1/d}-1)}
    =\Theta\!\left(N^{1-1/d}\right).
    \label{eq:depth-communication-separation}
\end{equation}
Part (ii) also gives the forward computation a sublinear first-arrival time: input information reaches every state within $d(m-1)$ rounds, although the longest input-to-state path has length up to $N-1$. The next subsection shows that the exponent in \cref{eq:depth-communication-separation} cannot be improved.

\subsection{Lower bound: a light cone for loss information}
\label{sec:lower-bound}

\begin{lemma}[Light cone]
\label{lem:light-cone}
After $K$ rounds of a $\rho$-local algorithm, the memory of site $v$ is a function of the round-$0$ memories of the sites within distance $K\rho$ of $v$. In particular, if $\dist(v,O)>K\rho$, the memory of $v$ after $K$ rounds does not depend on the target $y$.
\end{lemma}
\begin{proof}
Induction on $K$. The claim is trivial for $K=0$. If it holds for $K-1$, the round-$K$ memory of $v$ is a function of the round-$(K-1)$ memories of sites $u$ with $\dist(u,v)\le\rho$, each of which is a function of round-$0$ memories within distance $(K-1)\rho$ of $u$, hence within distance $K\rho$ of $v$. The second statement follows because $y$ enters only the round-$0$ memories of $O$.
\end{proof}

The geometric half of the lower bound is a statement about realizable circuits alone.

\begin{lemma}[Credit depth of realizable circuits]
\label{lem:realizable-depth}
Let $G$ be a realizable circuit of range $\rho$ on a substrate of density at most $C_d$ with $|O|$ outputs, and let $N\ge C_d|O|$. Then
\begin{equation}
    D_{\mathrm{reach}}(G,O)\ \ge\ \frac{1}{\rho}\left(\frac{N}{C_d|O|}\right)^{1/d},
    \label{eq:realizable-depth}
\end{equation}
and the same bound holds for the forward eccentricity of any set of $|O|$ input vertices.
\end{lemma}
\begin{proof}
Fix any $\varrho$ with $1\le\varrho<(N/(C_d|O|))^{1/d}$. The sites within distance $\varrho$ of $O$ number at most $|O|\,C_d\varrho^d<N$, so some vertex $v$ has $\dist(v,O)>\varrho$. Every edge of $G$ has length at most $\rho$, so every directed path from $v$ to $O$ has more than $\varrho/\rho$ edges. Taking $\varrho\uparrow(N/(C_d|O|))^{1/d}$ gives \cref{eq:realizable-depth}; the forward statement is identical with the roles of sources and targets exchanged.
\end{proof}

\Cref{lem:realizable-depth} bounds the credit depth of the graph. The theorem below bounds the rounds used by any algorithm on the substrate, whether or not it follows the edges of a fixed graph.

\begin{theorem}[Light-cone lower bound for local credit assignment]
\label{thm:light-cone}
Let a substrate have density at most $C_d$ and $|O|$ output sites, and let $\mathcal A$ be a $\rho$-local learning algorithm. Suppose that for a site $v$ there are two instances $(x,\theta,y)$ and $(x,\theta,y')$ with gradients $\nabla_{\theta_v}\ell=g\neq0$ and $-g$, respectively, at the parameter block held at $v$. If after $K$ rounds $\mathcal A$ holds at $v$ an estimate $\hat g_v$ with $\langle\hat g_v,\nabla_{\theta_v}\ell\rangle>0$ on both instances, then
\begin{equation}
    K\ \ge\ \frac{\dist(v,O)}{\rho}.
    \label{eq:light-cone-site}
\end{equation}
Moreover, if $N\ge C_d|O|$ then some site satisfies $\dist(v,O)\ge\bigl(N/(C_d|O|)\bigr)^{1/d}$. Hence an algorithm that guarantees sign-correct estimates at every loss-sensitive site needs
\begin{equation}
    K\ \ge\ \frac{1}{\rho}\left(\frac{N}{C_d|O|}\right)^{1/d}
    =\Omega\bigl(N^{1/d}\bigr)
    \label{eq:light-cone-global}
\end{equation}
rounds after the target is revealed, for fixed $d$, $\rho$, $C_d$, and $|O|$.
\end{theorem}

\begin{proof}
If $K\rho<\dist(v,O)$, \cref{lem:light-cone} gives identical memories at $v$ on the two instances, hence the same estimate $\hat g_v$; a single vector cannot have positive inner product with both $g$ and $-g$. This proves \cref{eq:light-cone-site}. The second claim is the density argument in the proof of \cref{lem:realizable-depth}, and \cref{eq:light-cone-global} follows from \cref{eq:light-cone-site}.
\end{proof}

\begin{remark}[Sign-flipped twins exist]
\label{rem:sign-flip}
The hypothesis of \cref{thm:light-cone} is satisfied at every loss-sensitive site for squared loss. With $\ell=\tfrac12\|\hat y-y\|^2$ and $\hat y$ a function of the settled state, the forcing $b=\nabla_{h_O}\ell$ is linear in $\hat y-y$, and by \cref{eq:settled-adjoint,eq:parameter-gradient-general} every parameter gradient is linear in $b$ once $(x,\theta)$ is fixed. Replacing $y$ by $y'=2\hat y-y$ negates $b$ and therefore negates every gradient. The same holds for any loss whose output gradient can be negated by a change of target.
\end{remark}

\begin{corollary}[Depth and credit depth of realizable circuits]
\label{cor:optimal-exponent}
For every family of circuits realizable in dimension $d$, the credit depth is $\Omega(N^{1/d})$, and no $\rho$-local learning algorithm on the family's substrates delivers sign-correct credit to every loss-sensitive parameter in $o(N^{1/d})$ rounds. The family $\{G_{m,d}\}$ of \cref{thm:spatial-compiler} is realizable in dimension $d$ with range $1$, degree $2d$, depth $N-1$, and credit depth $d(N^{1/d}-1)$. Depth and credit depth are therefore independent parameters of realizable circuits: for every $N$ there is a realizable circuit of maximal depth $N-1$ whose credit depth is $\Theta(N^{1/d})$, and no realizable circuit of any depth has credit depth $o(N^{1/d})$.
\end{corollary}

Both bounds concern causal availability only: a state that has become reachable may still hold inaccurate credit. \Cref{sec:finite-credit} quantifies the additional rounds in terms of the Jacobian-weighted path tail.

\subsection{The sparse Moore compiler}
\label{sec:moore}

The full grid of \cref{thm:spatial-compiler} uses every unit adjacency. The experiments use a sparser two-dimensional compiler that keeps the exponent with fewer edges. Let $m=2^k$ and $N=m^2$. A closed Moore curve \cite{Moore1900}, generated here by an L-system \cite{Prusinkiewicz1996}, visits every point of the $m\times m$ grid exactly once and returns to its start, so it is a Hamiltonian cycle $C$ of the unit grid. Cutting $C$ at one edge induces a topological order and the backbone $v_0\to v_1\to\dots\to v_{N-1}$. The cut edge is retained and oriented from input to output, and every unit-grid edge that separates two cells enclosed by $C$ is added and oriented from lower to higher index (\cref{alg:moore-dag}). \Cref{fig:moore-grid} shows the $4\times4$ and $8\times 8$ cases.

\begin{proposition}[Edge count of cycle-based compilers]
\label{prop:moore-edges}
Let $m$ be even and let $C$ be any Hamiltonian cycle of the $m\times m$ unit grid. The graph consisting of the $N-1$ backbone edges, the cut edge, and the unit-grid edges separating two cells enclosed by $C$ has exactly
\begin{equation}
    |E|=\tfrac32N-2
    \label{eq:moore-edge-count}
\end{equation}
edges, maximum undirected degree four, and unit-length edges only.
\end{proposition}
\begin{proof}
$C$ is a simple closed lattice polygon with all $N$ lattice points of the grid on its boundary and none in its interior. By Pick's theorem its area is $N/2-1$, so it encloses $N/2-1$ unit cells. Each of the $N$ edges of $C$ is a side of exactly one enclosed cell, and every other side of an enclosed cell is shared with another enclosed cell, since a side separating an enclosed cell from a non-enclosed cell or from the exterior of the grid would lie on the boundary of the enclosed region, which is $C$. Counting sides of enclosed cells, $4(N/2-1)=N+2s$ with $s$ the number of shared sides, so $s=N/2-2$, and $|E|=(N-1)+1+(N/2-2)$. Degree and length are inherited from the grid.
\end{proof}
The $N/2-2$ shortcut edges are the sides of a tree in the dual graph of the enclosed cells, which is why the compiler is sparse: it uses about three quarters of the $2N-2\sqrt N$ unit adjacencies of the grid.

\begin{figure}[t]
    \centering
    \begin{subfigure}[b]{0.48\textwidth}
        \centering
        \includegraphics[width=\textwidth]{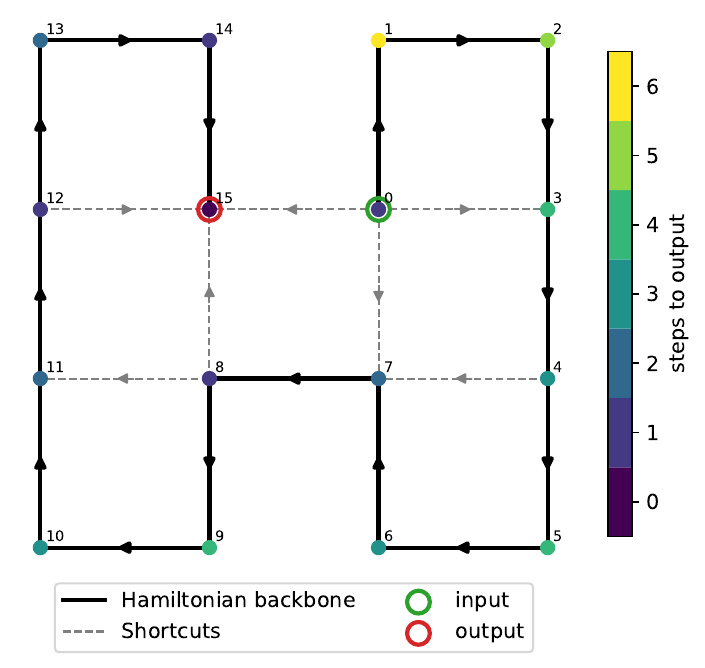}
        \caption{$4\times 4$ grid ($m=4, N=16$)}
    \end{subfigure}
    \hfill
    \begin{subfigure}[b]{0.48\textwidth}
        \centering
        \includegraphics[width=\textwidth]{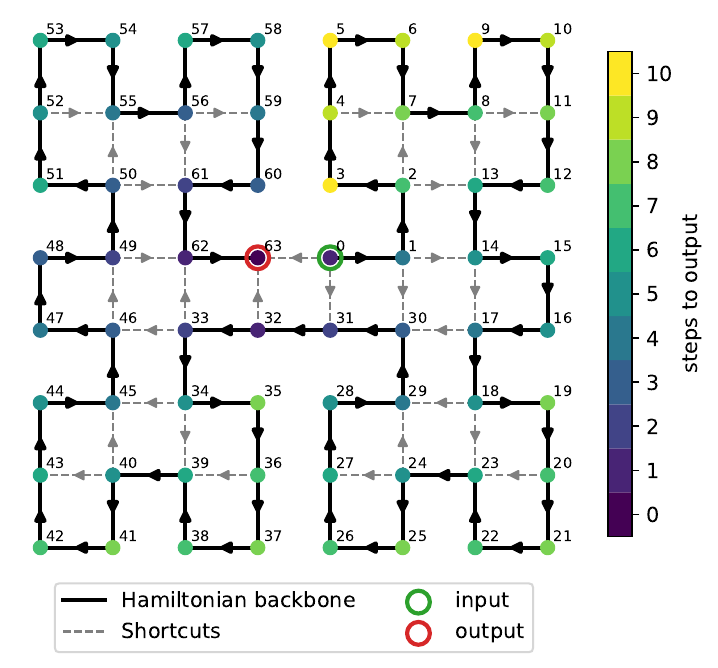}
        \caption{$8\times 8$ grid ($m=8,N=64$)}
    \end{subfigure}
    \caption{Sparse Moore compiler for two grids. Solid arrows form the Hamiltonian backbone; dashed arrows are the cut edge and the enclosed-cell shortcuts. Every edge is unit length and respects the common topological order. The state(s) farthest from the output is state~1 at distance $6$ for the $4\times 4$ case and states~3, 5, 9 at distance $10$ for the $8\times 8$ case. Both cases satisfy $D_\text{reach}=m+2=\sqrt N+2$.}
    \label{fig:moore-grid}
\end{figure}

For the implemented sizes $m\in\{4,8,16,32,64,128,256\}$, the reverse eccentricity of the output, computed exactly by breadth-first search, is
\begin{equation}
    D_{\mathrm{reach}}=m+2=\sqrt N+2 .
    \label{eq:moore-reach-empirical}
\end{equation}
\Cref{eq:moore-reach-empirical} is not proved for all $k$; the formal $O(\sqrt N)$ guarantee is \cref{thm:spatial-compiler}, which the Moore compiler matches in exponent with a smaller constant and fewer edges. The graphs of this section fix where credit can travel; the next section specifies the dynamics that run on them.

\section{SpatialPC: Local Dynamics on the Compiled Graph}
\label{sec:dynamics}

A compiled graph bounds how far information can travel per round; it does not say how the states are computed or how the gradient is obtained. SpatialPC runs synchronous local dynamics on the graph for a fixed number of rounds $T$, trains the resulting finite-round predictor, and obtains the exact gradient of that objective by APD on the unrolled computation, using the same edges in reverse. The separation of \cref{sec:spatialpc} is what allows $T$ to be of order $N^{1/d}$ rather than $N$.

\subsection{Synchronous forward dynamics}
\label{sec:forward}
For a directed edge $(u,v)$ the message sent from $u$ to $v$ is
\begin{equation}
    \mu_{uv}^t
    =s_{uv}h_u^t+g_{uv}W_{uv}\sigma(h_u^t),
    \label{eq:spatial-message}
\end{equation}
with degree-normalized skip gain and residual gain
\begin{equation}
    s_{uv}=\frac{1}{\sqrt{d_{\mathrm{in}}(v)\,d_{\mathrm{out}}(u)}},
    \qquad
    g_{uv}=\frac{\gamma_0}{\sqrt{mH}}\,s_{uv},
    \label{eq:gains}
\end{equation}
where $\gamma_0$ is a fixed constant, $m=\sqrt N$ is the grid side, and $H$ is the state width. Each state aggregates the messages of its parents, and the input $x\in\mathbb R^{d_x}$ is injected at the root $v_{\mathrm{in}}$,
\begin{equation}
    h_v^{t+1}=\sum_{u\in\pa(v)}\mu_{uv}^t
    \quad(v\neq v_{\mathrm{in}}),
    \qquad
    h^{t+1}_{v_{\mathrm{in}}}=W_{\mathrm{in}}x/\sqrt{d_x}.
    \label{eq:local-forward}
\end{equation}
Writing one simultaneous update of all states as $F_\theta$,
\begin{equation}
    h^{t+1}=F_\theta(h^t;x),
    \qquad h^0=0,
    \label{eq:forward-dynamics}
\end{equation}
and the settled state $h^\star$ is the unique solution of
\begin{equation}
    h^\star=F_\theta(h^\star;x),
    \label{eq:settled-forward}
\end{equation}
which a single sweep in topological order computes exactly. The prediction is $\hat y=\readout(h)=\kappa_H W_{\mathrm{out}}|O|^{-1/2}\sum_{o\in O}h_o$ (\cref{alg:spatialpc-forward-step}).

\begin{proposition}[Synchronous forward sweep]
\label{prop:forward-causality}
Let $L(v)$ be the length of the longest directed path from $v_{\mathrm{in}}$ to $v$. For every $T\ge1$:
\begin{enumerate}[label=(\roman*),leftmargin=*,itemsep=0pt,topsep=1pt]
    \item $h_v^T$ depends on $x$ only through directed paths from $v_{\mathrm{in}}$ to $v$ of length at most $T-1$;
    \item $h_v^T=h_v^\star$ whenever $T\ge L(v)+1$; in particular $h^T=h^\star$ for all $T\ge D_{\mathrm{seq}}+1$.
\end{enumerate}
\end{proposition}
\begin{proof}
(i) $h^1$ depends on $x$ only at $v_{\mathrm{in}}$, and each application of $F_\theta$ moves dependence across one edge. (ii) Induction on $t$: $h^1_{v_{\mathrm{in}}}=h^\star_{v_{\mathrm{in}}}$, and if $h^t_u=h^\star_u$ for all $u$ with $L(u)\le t-1$, then for every $v$ with $L(v)\le t$ all parents satisfy $L(u)\le t-1$, so $h^{t+1}_v=\sum_{u\in\pa(v)}\mu_{uv}(h^\star_u)=h^\star_v$. Since every state lies on the backbone, $\max_vL(v)=D_{\mathrm{seq}}$.
\end{proof}

\Cref{prop:forward-causality} is the forward counterpart of \cref{lem:light-cone}: the synchronous iteration is a Jacobi sweep of the triangular system \cref{eq:settled-forward}, exact only after $D_{\mathrm{seq}}+1$ rounds and truncated at path length $T-1$ before that. With $T<D_{\mathrm{seq}}+1$ the finite-round state is not the settled state. Rather than wait for it to settle, SpatialPC trains the parameters so that the $T$-round state is nearly a fixed point.

\subsection{Finite-round objective}
\label{sec:objective}
The finite-round predictor is trained with the terminal feasibility residual
\begin{equation}
    r^T=h^T-F_\theta(h^T;x)
    \label{eq:feasibility-residual}
\end{equation}
and the objective
\begin{equation}
    \mathcal J_T(\theta;x,y)
    =\ell\bigl(\readout(h^T),y\bigr)
     +\frac{\lambda}{2NH}\|r^T\|_F^2 .
    \label{eq:spatial-objective}
\end{equation}
The residual term is the predictive-coding energy \cref{eq:pc-loss} evaluated at the finite-round state; it is zero exactly when $h^T$ is a fixed point of $F_\theta$, that is, when the $T$-round computation is already settled. The gradient of $\mathcal J_T$ is obtained next, by the same local mechanism that \cref{sec:apd} introduced for the settled problem.

\subsection{Exact gradients by APD on the space--time graph}
\label{sec:finite-apd}
Unrolling \cref{eq:forward-dynamics} gives the time-indexed constraints
\begin{equation}
    h^{t+1}-F_\theta(h^t;x)=0,
    \qquad t=0,\dots,T-1,
    \label{eq:unrolled-constraint}
\end{equation}
whose constraint graph is $T$ copies of the spatial graph stacked in time. Let $J_t=\partial F_\theta(h^t;x)/\partial h^t$. APD on this space--time graph has the terminal adjoint
\begin{equation}
    a^T
    =\nabla_{h^T}\ell
     +\frac{\lambda}{NH}\bigl(r^T-J_T^\top r^T\bigr)
    \label{eq:terminal-apd}
\end{equation}
and the reverse recursion
\begin{equation}
    a^t=J_t^\top a^{t+1},
    \qquad t=T-1,\dots,0 .
    \label{eq:finite-apd-recursion}
\end{equation}
Each product $J_t^\top a^{t+1}$ decomposes into child-to-parent messages along the edges used in the forward pass (\cref{alg:spatialpc-training}).

\begin{proposition}[Exactness and locality of finite-round APD]
\label{prop:finite-spatialpc-apd}
The recursion \cref{eq:terminal-apd,eq:finite-apd-recursion} yields the exact gradient of \cref{eq:spatial-objective},
\begin{equation}
    \nabla_\theta\mathcal J_T
    =\left.\nabla_\theta\mathcal J_T\right|_{h^T\ \mathrm{fixed}}
     +\sum_{t=0}^{T-1}
      \left(\frac{\partial F_\theta(h^t;x)}{\partial\theta}\right)^{\top}a^{t+1},
    \label{eq:finite-apd-param}
\end{equation}
and every operation in it is local to an edge of the spatial graph: a forward message, a reverse message, or an outer product of a postsynaptic adjoint with a presynaptic activity.
\end{proposition}
\begin{proof}
Apply the Lagrangian \cref{eq:pc-lagrangian} to the constraints \cref{eq:unrolled-constraint}. Stationarity in $h^T$ gives \cref{eq:terminal-apd}, stationarity in $h^t$ for $t<T$ gives \cref{eq:finite-apd-recursion}, and differentiating in $\theta$ gives \cref{eq:finite-apd-param}; \cref{app:training-derivation} carries out the computation. Locality holds because $F_\theta$ is a sum of edge messages, so $J_t^\top$ and $\partial F_\theta/\partial\theta$ are sums of edge-local blocks.
\end{proof}

\Cref{prop:finite-spatialpc-apd} is the adjoint method applied to the unrolled computation; the adjoint pass needs no communication beyond the edges already present. With $\lambda>0$ one parameter update uses $T$ forward message phases, one residual evaluation, one pullback for $J_T^\top r^T$, and $T$ reverse phases, for $2T+2$ local message phases in a strictly serial accounting. The implementation is verified against automatic differentiation to double precision in \cref{app:gradient-verification}.

\subsection{Relation to predictive coding}
\label{sec:relation-to-pc}
SpatialPC is predictive coding in the constrained sense of \cref{eq:constrained-pc}. The objective contains the PC energy evaluated at the finite-round state, every operation is local to an edge, each weight update has the form of a postsynaptic credit signal times a presynaptic activity, and the settled adjoint equation \cref{eq:settled-adjoint} is the state-stationarity condition of augmented-Lagrangian PC, with $a$ equal to minus the augmented multiplier (\cref{app:apd-pcalm}).

It is not a relaxation-based PC network. There is no inference phase: states are produced by $T$ synchronous rounds from $h^0=0$ with nothing clamped, the target never acts on the states, and credit is carried by a separate adjoint variable rather than by prediction errors relaxing under a clamped output. The reverse recursion \cref{eq:finite-apd-recursion} is reverse-mode differentiation of the unrolled computation, so APD computes the gradient of \cref{eq:spatial-objective} exactly rather than approximating it at an inference equilibrium, and the feasibility term uses the PC energy as a penalty that pulls the finite-round state toward a fixed point of $F_\theta$, not as the dynamics. What SpatialPC changes is the graph on which exact local credit is propagated, not the local rule. The separation of \cref{sec:spatialpc} and the analysis of \cref{sec:finite-credit} apply equally to relaxation-based PC dynamics run on the same compiled graph, since the light-cone bound constrains every local algorithm and the settled adjoint is the fixed point those dynamics approach.

The gradient of \cref{prop:finite-spatialpc-apd} is exact for the finite-round objective. How close the finite-round credit is to the credit of the fully settled computation is a separate question, and it is the subject of the next section.

\section{Credit Fidelity Beyond Reachability}
\label{sec:finite-credit}

\Cref{sec:spatialpc} determines when loss information can reach a state; \cref{sec:dynamics} gives an exact gradient for a fixed round budget. What remains is the relation between the two: how many rounds the finite-round credit needs before it agrees with the credit of the fully settled computation. The accuracy of arriving credit depends on every reverse path, not only on the shortest one. On a settled graph APD makes this dependence explicit, and the identities below separate a geometric quantity, reachability, from a dynamical one, the decay of Jacobian powers.

\subsection{Finite reverse-path expansion}
\label{sec:reverse-expansion}
Fix a settled state $h^\star$ and set
\begin{equation}
    J=\left.\frac{\partial F_\theta(h;x)}{\partial h}\right|_{h=h^\star},
    \qquad
    b=\nabla_h\ell(h^\star_O,y).
\end{equation}
The settled adjoint solves $a^\star=b+J^\top a^\star$, and synchronous APD from $a^0=0$ is
\begin{equation}
    a^{k+1}=b+J^\top a^k .
    \label{eq:apd-jacobi}
\end{equation}
In a topological order $J$ is strictly block lower triangular, and $(J^k)_{wv}\neq0$ requires a directed path of length $k$ from $v$ to $w$. Hence $J^{D+1}=0$ with $D=D_{\mathrm{seq}}(G)$.

\begin{proposition}[Reverse-path expansion and tail identity]
\label{prop:finite-neumann}
For every $K\ge0$,
\begin{equation}
    a^K=\sum_{k=0}^{K-1}(J^\top)^k b,
    \qquad
    a^\star=(I-J^\top)^{-1}b=\sum_{k=0}^{D}(J^\top)^k b,
    \label{eq:reverse-path-expansion}
\end{equation}
and the truncation error is
\begin{equation}
    a^\star-a^K
    =\sum_{k=K}^{D}(J^\top)^k b
    =(J^\top)^K a^\star .
    \label{eq:apd-tail}
\end{equation}
In particular $a^K=a^\star$ for all $K\ge D+1$.
\end{proposition}
\begin{proof}
The first identity follows from \cref{eq:apd-jacobi} by induction. Since $J^\top$ is nilpotent, $I-J^\top$ is invertible with inverse $\sum_{k\le D}(J^\top)^k$, which gives the second. Subtracting, $a^\star-a^K=\sum_{k=K}^{D}(J^\top)^kb=(J^\top)^K\sum_{j=0}^{D-K}(J^\top)^jb=(J^\top)^Ka^\star$, using $(J^\top)^jb=0$ for $j>D$.
\end{proof}

The $k$th term $(J^\top)^kb$ sums the Jacobian products along all reverse paths of length exactly $k$ from the output, so APD accumulates credit shell by shell, and what it has not yet delivered after $K$ rounds is the settled adjoint pushed back through $K$ further edges. Two consequences follow: a causal one, which depends only on the graph, and a quantitative one, which depends on $J$.

\subsection{Causal arrival and the effective credit radius}
\label{sec:credit-radius}
Let $d_O(v)$ be the length of the shortest directed path from $v$ to an output node on which $b$ is nonzero.

\begin{proposition}[Credit causal cone]
\label{prop:credit-causal-cone}
$a_v^K=0$ whenever $d_O(v)\ge K$. Output-dependent credit therefore cannot appear at $v$ before the $(d_O(v)+1)$st APD iterate, and $D_{\mathrm{reach}}+1$ iterates are necessary before every state can carry nonzero credit.
\end{proposition}
\begin{proof}
Every term of $a^K_v$ in \cref{eq:reverse-path-expansion} corresponds to a reverse path of length at most $K-1$ from a node in the support of $b$ to $v$, and no such path exists when $d_O(v)\ge K$.
\end{proof}

\Cref{prop:credit-causal-cone} is the APD form of \cref{lem:light-cone}, with graph distance in place of physical distance. The rounds required beyond causal arrival are governed by the tail. Define the normalized omitted credit mass and the \emph{effective credit radius}
\begin{equation}
    \tau_K(J,b)
    :=\frac{\|a^\star-a^K\|}{\|a^\star\|}
    =\frac{\|(J^\top)^Ka^\star\|}{\|a^\star\|},
    \qquad
    R_\epsilon(J,b)
    :=\min\{K:\tau_K(J,b)\le\epsilon\}.
    \label{eq:effective-credit-radius}
\end{equation}
$D_{\mathrm{reach}}$ depends on the graph and the output placement only; $R_\epsilon$ depends in addition on the learned Jacobian and on the loss forcing. The two can be bounded independently.

\begin{proposition}[Tail bounds]
\label{prop:tail-bounds}
For every $K$, $\tau_K(J,b)\le\|J^K\|_2\le\|J\|_2^K$. Consequently
\begin{equation}
    R_\epsilon(J,b)\ \le\ \min\bigl\{K:\|J^K\|_2\le\epsilon\bigr\}\ \le\ D+1,
    \label{eq:radius-bound}
\end{equation}
and if $\|J\|_2\le q<1$ then $R_\epsilon(J,b)\le\lceil\log(1/\epsilon)/\log(1/q)\rceil$.
\end{proposition}
\begin{proof}
Immediate from \cref{eq:apd-tail}, submultiplicativity, and $J^{D+1}=0$.
\end{proof}

Contraction is sufficient for a fast tail but not necessary: because $J$ is nilpotent, $\|J^K\|_2$ vanishes at $K=D+1$ whatever the value of $\|J\|_2$, and it can become small well before that even when $\|J\|_2>1$. What governs the tail is the decay of the Jacobian powers along the settled adjoint, which \cref{sec:experiments} measures through the alignment of finite-round and settled gradients.

\subsection{From adjoint error to update fidelity}
\label{sec:angle}
For a local linear parameter block $W_{uv}$ at a settled state, the exact and $K$-round gradients share the presynaptic factor,
\begin{equation}
    g_{uv}^\star=\kappa\,a_v^\star z_u^\top,
    \qquad
    g_{uv}^K=\kappa\,a_v^K z_u^\top,
    \qquad z_u\neq0 .
\end{equation}

\begin{corollary}[Local gradient-angle control]
\label{cor:gradient-angle}
If $\|a_v^K-a_v^\star\|\le\epsilon\|a_v^\star\|$ with $0\le\epsilon<1$, then
\begin{equation}
    \cos\angle\bigl(g_{uv}^K,g_{uv}^\star\bigr)
    =\cos\angle\bigl(a_v^K,a_v^\star\bigr)
    \ge\sqrt{1-\epsilon^2},
    \label{eq:gradient-angle-bound}
\end{equation}
that is, the angle between the local updates is at most $\arcsin\epsilon$.
\end{corollary}
The proof is in \cref{app:angle-proof}; the bound is attained when $a_v^K$ is tangent to the ball of radius $\epsilon\|a_v^\star\|$ about $a_v^\star$. With \cref{eq:effective-credit-radius}, $K\ge R_\epsilon$ rounds give every local update whose adjoint error is at most the global tolerance an angle of at most $\arcsin\epsilon$ from the settled update.

\subsection{Forward--reverse duality and the residual as a certificate}
\label{sec:duality}
The reverse expansion has a forward counterpart, which connects the feasibility residual penalized in \cref{eq:spatial-objective} to the distance between the finite-round and settled states.
\begin{proposition}[Forward truncation and residual identity]
\label{prop:forward-reverse}
Let the local dynamics be affine, $F(h)=c+Jh$, with $h^0=0$. For every $T\ge0$,
\begin{equation}
    h^T=\sum_{k=0}^{T-1}J^k c,
    \qquad
    h^\star=\sum_{k=0}^{D}J^k c,
    \qquad
    h^\star-h^T=J^T h^\star,
    \label{eq:forward-path-expansion}
\end{equation}
the terminal residual is the first omitted forward shell,
\begin{equation}
    r^T=h^T-F(h^T)=-J^T c,
    \label{eq:residual-shell}
\end{equation}
and the distance to the settled state is determined by the residual,
\begin{equation}
    h^\star-h^T=-(I-J)^{-1}r^T,
    \qquad
    \|h^\star-h^T\|\le\Bigl(\sum_{k=0}^{D}\|J^k\|_2\Bigr)\|r^T\| .
    \label{eq:residual-certificate}
\end{equation}
\end{proposition}
\begin{proof}
The expansions follow from $h^{t+1}=c+Jh^t$ by induction and from nilpotency as in \cref{prop:finite-neumann}. Then $r^T=\sum_{k<T}J^kc-c-\sum_{k=1}^{T}J^kc=-J^Tc$, and $(I-J)(h^\star-h^T)=\sum_{k=T}^{D}J^kc-\sum_{k=T+1}^{D+1}J^kc=J^Tc=-r^T$.
\end{proof}

Forward inference accumulates forward paths by length and APD accumulates reverse paths by length; the truncation errors $J^Th^\star$ and $(J^\top)^Ka^\star$ are governed by the same Jacobian powers. The feasibility penalty acts on $r^T$, which by \cref{eq:residual-certificate} is a computable certificate for the distance between the finite-round state and the settled state; for the nonlinear dynamics used in practice the identities hold to first order about $h^T$, which is the regime in which the residual is small. This is the mechanism by which the feasibility weight should control credit accuracy, and the next section tests it, together with the structural separation and the dependence of fidelity on the round budget.

\section{Experiments}
\label{sec:experiments}

This section now empirically validates three testable claims: that the implemented compiler realizes the structural separation at scale (\cref{thm:spatial-compiler,prop:moore-edges}); that a round budget of order $\sqrt N$ delivers gradients aligned with the settled gradients; and that beyond causal arrival, credit accuracy is governed by the feasibility residual and the round budget (\cref{prop:forward-reverse,prop:tail-bounds}) rather than by reachability. \Cref{sec:structural-scaling} tests the first with structural measurements to $N=65{,}536$; \cref{sec:state-scaling,sec:lambda-ablation,sec:round-ablation} test the second and third with task experiments at $N\le256$; \cref{sec:longrun} reports training over twenty epochs.

\subsection{Setup}
\label{sec:setup}
\paragraph{Data and optimization.}
MNIST \cite{mnist} and Fashion-MNIST \cite{fashionmnist} are used with their standard 60{,}000/10{,}000 splits, flattened to $d_x=784$ and normalized, without augmentation. All results are mean $\pm$ standard deviation over three seeds. SpatialPC is trained with Adam \cite{kingma2017adam} at learning rate $10^{-3}$ from i.i.d.\ $\mathcal N(0,1)$ base weights.

\paragraph{Reference configuration.}
The reference model is a $16\times16$ Moore graph ($N=256$, 382 learned state-to-state matrices), width $H=128$, ReLU activations, $T=32$ rounds, and feasibility weight $\lambda=100$. Two readout protocols are reported: \emph{SpatialPC-v1} uses cross-entropy with readout scale $\kappa_H=H^{-1/2}$, and \emph{SpatialPC-v2} uses one-hot mean-squared error with $\kappa_H=H^{-1}$.

\paragraph{Measured quantities.}
\emph{Finite-round accuracy} is computed from $h^T$. \emph{Settled accuracy} uses the same parameters after an exact topological sweep. \emph{Gradient cosine} is the cosine between the finite-round APD task gradient (the $\ell$-dependent part of \cref{eq:finite-apd-param}, excluding the direct feasibility term) and the settled task gradient $\nabla_\theta\ell(\readout(h^\star),y)$ of the same network, so it measures how closely the $T$-round credit matches the credit of the fully settled computation. \emph{Residual RMS} is the root-mean-square entry of the terminal residual $r^T$, the quantity penalized in \cref{eq:spatial-objective}.

\subsection{Structural scaling to 65{,}536 states}
\label{sec:structural-scaling}

Moore graphs were constructed for $N\in\{16,64,256,1024,4096,16384,65536\}$, and $|E|$, $D_{\mathrm{seq}}$, and $D_{\mathrm{reach}}$ were computed exactly.
The separation grows from $2.5\times$ at $N=16$ to $254\times$ at $N=65{,}536$ (\cref{fig:structural-scaling,tab:structural-scaling}), following the $N^{1/2}$ exponent of \cref{cor:optimal-exponent}. At $N=4{,}096$ a 4{,}095-edge backbone has worst-node output distance 66; at $N=65{,}536$ a 65{,}535-edge backbone has worst-node output distance 258. The edge counts match \cref{prop:moore-edges} at every size. The remaining experiments ask whether a round budget on this scale yields usable gradients.

\begin{figure}[t]
    \centering
    \includegraphics[width=0.72\linewidth]{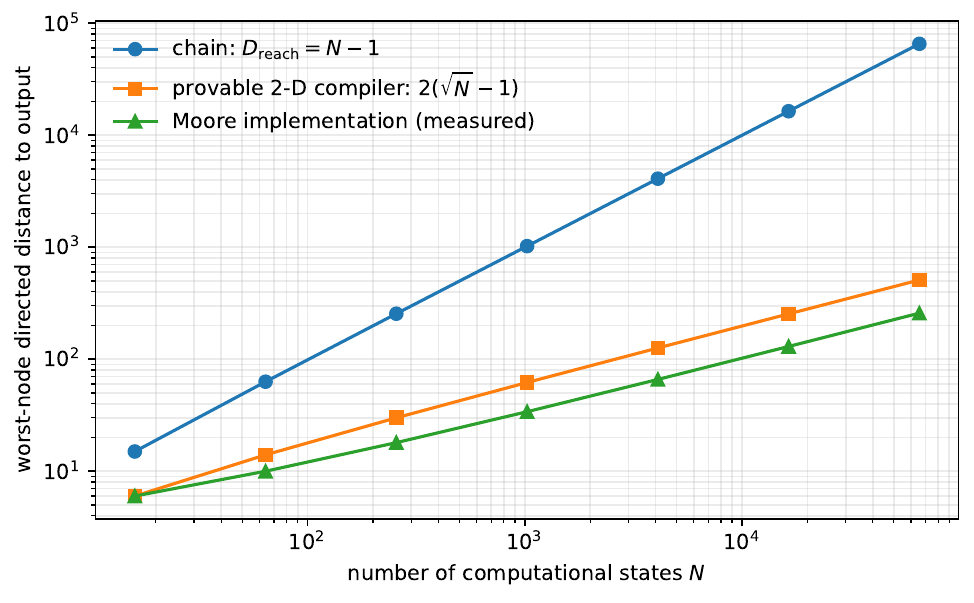}
    \caption{Represented depth and reachability depth versus state count. A chain has $D_{\mathrm{reach}}=N-1$; the snake-ordered grid of \cref{thm:spatial-compiler} has $D_{\mathrm{reach}}=2(\sqrt N-1)$; the Moore compiler follows $D_{\mathrm{reach}}=\sqrt N+2$ over the measured sizes, with $D_{\mathrm{seq}}=N-1$ in all cases.}
    \label{fig:structural-scaling}
\end{figure}

\begin{table}[htbp]
\centering
\caption{Structural scaling of the Moore compiler. $|E|=\tfrac32N-2$ (\cref{prop:moore-edges}), $D_{\mathrm{seq}}=N-1$ is the backbone length, $D_{\mathrm{reach}}=\sqrt N+2$ is the exact worst-node distance to the output, and the last column is $D_{\mathrm{seq}}/D_{\mathrm{reach}}$.}
\label{tab:structural-scaling}
\small
\begin{tabular}{rrrrrr}
\toprule
$\bm{m}$ & $\bm{N=m^2}$ & $\bm{|E|}$ & $\bm{D_{\mathrm{seq}}}$ & $\bm{D_{\mathrm{reach}}}$ & \textbf{separation} \\
\midrule
4   & 16     & 22     & 15     & 6   & $2.5\times$ \\
8   & 64     & 94     & 63     & 10  & $6.3\times$ \\
16  & 256    & 382    & 255    & 18  & $14.2\times$ \\
32  & 1,024  & 1,534  & 1,023  & 34  & $30.1\times$ \\
64  & 4,096  & 6,142  & 4,095  & 66  & $62.0\times$ \\
128 & 16,384 & 24,574 & 16,383 & 130 & $126.0\times$ \\
256 & 65,536 & 98,302 & 65,535 & 258 & $254.0\times$ \\
\bottomrule
\end{tabular}
\end{table}

\subsection{Gradient fidelity under a geometric round budget}
\label{sec:state-scaling}

The number of states is increased with the round budget tied to the reachability scale,
\begin{equation}
    T=2\sqrt N,
    \qquad
    (N,T)\in\{(16,8),(64,16),(256,32)\}.
    \label{eq:geometric-budget}
\end{equation}

Settled-gradient alignment of SpatialPC-v1 does not degrade as the backbone grows while the budget stays at $2\sqrt N$ (\cref{tab:ablation-state-scaling}). On MNIST the gradient cosine is $0.804$, $0.927$, and $0.928$ at $N=16$, $64$, and $256$, and test accuracy rises from $91.71\%$ to $92.75\%$; on Fashion-MNIST the cosine is $0.860$, $0.937$, and $0.974$, with accuracy above $82\%$ throughout. At $N=256$ the forward computation receives 32 synchronous rounds although the backbone has 255 edges. SpatialPC-v2 also improves in accuracy with $N$ but its MNIST gradient cosine drops to $0.647$ at $N=256$; v1 is therefore the principal configuration below, and v2 is retained to show the sensitivity of credit fidelity to the loss and readout scaling.

\begin{table}[htbp]
\caption{Gradient fidelity under the geometric budget $T=2\sqrt N$ as the number of states grows ($H=128$, $\lambda=100$, one epoch). Gradient cosine compares the finite-round APD task gradient with the settled task gradient of the same network.}
\label{tab:ablation-state-scaling}
\centering
\small
\begin{tabular}{@{} llrr *{2}{r@{\,{\scriptsize$\pm$}\,}>{\scriptsize}l} @{}} 
\toprule
\textbf{Dataset} & \textbf{Method} & \textbf{\# of States} & \textbf{T} & \multicolumn{2}{c}{\textbf{Test Acc. (\%) $\uparrow$}} & \multicolumn{2}{c}{\textbf{Task-grad cos. to settled} $\uparrow$} \\
\midrule
MNIST         & SpatialPC-v1 &  16 &  8 & 91.71 & 0.14 & 0.8040 & 0.0122 \\
              &              &  64 & 16 & 92.27 & 0.08 & 0.9271 & 0.0139 \\
              &              & 256 & 32 & 92.75 & 0.23 & 0.9283 & 0.0086 \\
\addlinespace
              & SpatialPC-v2 &  16 &  8 & 84.44 & 0.97 & 0.8379 & 0.0116 \\
              &              &  64 & 16 & 89.90 & 0.23 & 0.9419 & 0.0082 \\
              &              & 256 & 32 & 91.22 & 0.27 & 0.6467 & 0.0279 \\
\midrule
Fashion-MNIST & SpatialPC-v1 &  16 &  8 & 82.24 & 0.12 & 0.8597 & 0.0230 \\
              &              &  64 & 16 & 82.73 & 0.12 & 0.9372 & 0.0038 \\
              &              & 256 & 32 & 82.85 & 0.18 & 0.9745 & 0.0008 \\
\addlinespace
              & SpatialPC-v2 &  16 &  8 & 75.36 & 0.21 & 0.8596 & 0.0045 \\
              &              &  64 & 16 & 79.55 & 0.40 & 0.9607 & 0.0051 \\
              &              & 256 & 32 & 80.66 & 0.18 & 0.8795 & 0.0231 \\
\bottomrule
\end{tabular}
\end{table}

\subsection{Feasibility controls credit accuracy}
\label{sec:lambda-ablation}

At $N=256$ and $T=32$ the feasibility weight is varied over $\lambda\in\{0,1,10,30,100\}$.
For MNIST v1, raising $\lambda$ from $0$ to $100$ lowers residual RMS from $0.0245$ to $0.0079$ and raises the gradient cosine from $0.353$ to $0.928$, while accuracy moves only from $92.67\%$ to $92.75\%$ (\cref{tab:ablation-lambda}). Fashion-MNIST behaves the same way: residual RMS falls from $0.0180$ to $0.0059$, the cosine rises from $0.628$ to $0.974$, and accuracy is unchanged within noise. Prediction quality thus saturates before the finite-round state yields accurate settled credit, and the feasibility term closes the gap, as \cref{prop:forward-reverse} predicts: the penalty acts on $r^T=-J^Tc$, the first omitted forward shell, and a small residual certifies a small distance to the settled state. The round budget is the second quantity identified in \cref{sec:finite-credit}.

\begin{table}[htbp]
\caption{Effect of the feasibility weight $\lambda$ at $N=256$, $T=32$, $H=128$ (one epoch). Accuracy is nearly constant while residual RMS falls and gradient cosine rises.}
\label{tab:ablation-lambda}
\centering
\small
\begin{tabular}{@{} llr *{3}{r@{\,{\scriptsize$\pm$}\,}>{\scriptsize}l} @{}} 
\toprule
\textbf{Dataset} & \textbf{Method} & $\bm{\lambda}$ & \multicolumn{2}{c}{\textbf{Test Acc. (\%)} $\uparrow$} & \multicolumn{2}{c}{\textbf{Task-grad cos. to settled} $\uparrow$} & \multicolumn{2}{c}{\textbf{Residual RMS} $\downarrow$} \\
\midrule
MNIST         & SpatialPC-v1 &   0 & 92.67 & 0.19 & 0.3525 & 0.1812 & 0.0245 & 0.0040 \\
              &              &   1 & 92.68 & 0.20 & 0.5057 & 0.2094 & 0.0204 & 0.0025 \\
              &              &  10 & 92.69 & 0.21 & 0.7688 & 0.0708 & 0.0144 & 0.0012 \\
              &              &  30 & 92.73 & 0.21 & 0.8637 & 0.0253 & 0.0112 & 0.0007 \\
              &              & 100 & 92.75 & 0.23 & 0.9283 & 0.0086 & 0.0079 & 0.0003 \\
\addlinespace
              & SpatialPC-v2 &   0 & 91.27 & 0.22 & 0.0872 & 0.0431 & 0.0272 & 0.0025 \\
              &              &   1 & 91.26 & 0.23 & 0.1105 & 0.0403 & 0.0242 & 0.0016 \\
              &              &  10 & 91.25 & 0.22 & 0.2297 & 0.0435 & 0.0168 & 0.0006 \\
              &              &  30 & 91.26 & 0.24 & 0.3800 & 0.0413 & 0.0125 & 0.0005 \\
              &              & 100 & 91.22 & 0.27 & 0.6467 & 0.0279 & 0.0083 & 0.0004 \\
\midrule
Fashion-MNIST & SpatialPC-v1 &   0 & 82.80 & 0.20 & 0.6282 & 0.2684 & 0.0180 & 0.0010 \\
              &              &   1 & 82.81 & 0.20 & 0.7676 & 0.1426 & 0.0153 & 0.0006 \\
              &              &  10 & 82.82 & 0.21 & 0.9157 & 0.0146 & 0.0107 & 0.0003 \\
              &              &  30 & 82.84 & 0.23 & 0.9509 & 0.0051 & 0.0083 & 0.0003 \\
              &              & 100 & 82.85 & 0.18 & 0.9745 & 0.0008 & 0.0059 & 0.0002 \\
\addlinespace
              & SpatialPC-v2 &   0 & 80.70 & 0.10 & 0.2182 & 0.0993 & 0.0150 & 0.0016 \\
              &              &   1 & 80.70 & 0.09 & 0.2755 & 0.1058 & 0.0138 & 0.0014 \\
              &              &  10 & 80.68 & 0.12 & 0.5058 & 0.0922 & 0.0102 & 0.0008 \\
              &              &  30 & 80.66 & 0.17 & 0.6935 & 0.0616 & 0.0079 & 0.0005 \\
              &              & 100 & 80.66 & 0.18 & 0.8795 & 0.0231 & 0.0053 & 0.0002 \\
\bottomrule
\end{tabular}
\end{table}

\subsection{Fidelity continues to improve past causal arrival}
\label{sec:round-ablation}

At $N=256$ the round budget is varied over $T\in\{8,16,32,64\}$. Residual RMS decreases monotonically with $T$ (\cref{tab:ablation-rounds}). On Fashion-MNIST v1 the gradient cosine rises from $0.795$ at $T=8$ to $0.990$ at $T=64$; on MNIST v1 it reaches $0.945$ at $T=16$ and $0.955$ at $T=64$, with a small non-monotonicity at $T=32$. The Moore graph has $D_{\mathrm{reach}}=18$ at $N=256$, so every state is reachable from $T=19$ onward, yet fidelity keeps improving through $T=64$. This is the gap between $D_{\mathrm{reach}}$ and $R_\epsilon(J,b)$ in \cref{eq:effective-credit-radius}: reachability fixes when credit can arrive, and the Jacobian-weighted tail fixes when it becomes accurate. The one-epoch results above leave open whether the gap stays controlled during training.

\begin{table}[htbp]
\caption{Effect of the round budget $T$ at $N=256$, $H=128$, $\lambda=100$ (one epoch). Every state is reachable once $T>D_{\mathrm{reach}}=18$, yet gradient fidelity keeps improving to $T=64$.}
\label{tab:ablation-rounds}
\centering
\small
\begin{tabular}{@{} llr *{3}{r@{\,{\scriptsize$\pm$}\,}>{\scriptsize}l} @{}} 
\toprule
\textbf{Dataset} & \textbf{Method} & $\bm{\lambda}$ & \multicolumn{2}{c}{\textbf{Test Acc. (\%)} $\uparrow$} & \multicolumn{2}{c}{\textbf{Task-grad cos. to settled} $\uparrow$} & \multicolumn{2}{c}{\textbf{Residual RMS} $\downarrow$} \\
\midrule
MNIST         & SpatialPC-v1 &  8 & 91.12 & 0.14 & 0.7009 & 0.0126 & 0.0249 & 0.0001 \\
              &              & 16 & 92.16 & 0.24 & 0.9451 & 0.0036 & 0.0125 & 0.0001 \\
              &              & 32 & 92.75 & 0.23 & 0.9283 & 0.0086 & 0.0079 & 0.0003 \\
              &              & 64 & 93.15 & 0.21 & 0.9550 & 0.0131 & 0.0047 & 0.0005 \\
\addlinespace
              & SpatialPC-v2 &  8 & 86.69 & 0.07 & 0.4624 & 0.0103 & 0.0169 & 0.0002 \\
              &              & 16 & 89.95 & 0.17 & 0.7785 & 0.0186 & 0.0106 & 0.0000 \\
              &              & 32 & 91.22 & 0.27 & 0.6467 & 0.0279 & 0.0083 & 0.0004 \\
              &              & 64 & 92.16 & 0.19 & 0.7794 & 0.0560 & 0.0061 & 0.0003 \\
\midrule
Fashion-MNIST & SpatialPC-v1 &  8 & 81.29 & 0.23 & 0.7946 & 0.0280 & 0.0199 & 0.0002 \\
              &              & 16 & 82.41 & 0.07 & 0.9689 & 0.0120 & 0.0102 & 0.0002 \\
              &              & 32 & 82.85 & 0.18 & 0.9745 & 0.0008 & 0.0059 & 0.0002 \\
              &              & 64 & 82.99 & 0.47 & 0.9904 & 0.0020 & 0.0031 & 0.0002 \\
\addlinespace
              & SpatialPC-v2 &  8 & 76.29 & 0.26 & 0.5018 & 0.1293 & 0.0145 & 0.0002 \\
              &              & 16 & 79.60 & 0.26 & 0.8684 & 0.0155 & 0.0086 & 0.0002 \\
              &              & 32 & 80.66 & 0.18 & 0.8795 & 0.0231 & 0.0053 & 0.0002 \\
              &              & 64 & 81.49 & 0.08 & 0.9246 & 0.0401 & 0.0036 & 0.0001 \\
\bottomrule
\end{tabular}
\end{table}

\subsection{Twenty epochs of training}
\label{sec:longrun}

The $N=256$, $T=32$, $H=128$ models are trained for 20 epochs and evaluated at epochs 1, 5, 10, and 20, alongside a depth-128, width-128 chain trained by backpropagation as an external reference.

After 20 epochs SpatialPC-v1 reaches $97.64\%$ on MNIST and $88.10\%$ on Fashion-MNIST, against $97.55\%$ and $88.03\%$ for the backpropagation reference (\cref{fig:exp2,tab:exp2}). The task-gradient cosines at epoch 20 are $0.883$ and $0.920$. Finite-round and settled accuracies stay within one point of each other throughout, while the gradient cosine drifts downward during training even though the residual stays flat or decreases: learning reshapes the weighted reverse-path distribution, and hence $R_\epsilon$, independently of the forward residual.

\begin{figure}[htbp]
    \centering
    \begin{subfigure}[b]{0.49\linewidth}
        \centering
        \includegraphics[width=\linewidth]{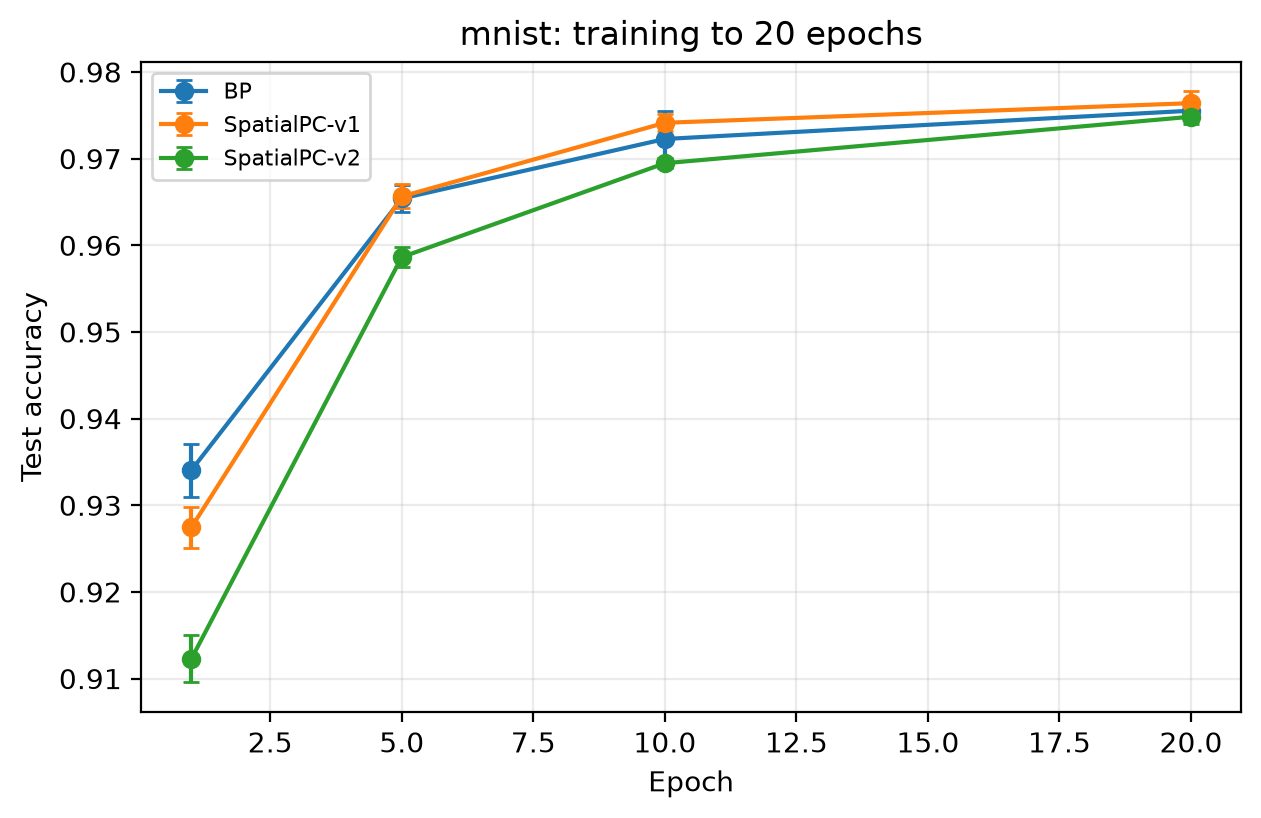}
        \caption{MNIST accuracy}
    \end{subfigure}\hfill
    \begin{subfigure}[b]{0.49\linewidth}
        \centering
        \includegraphics[width=\linewidth]{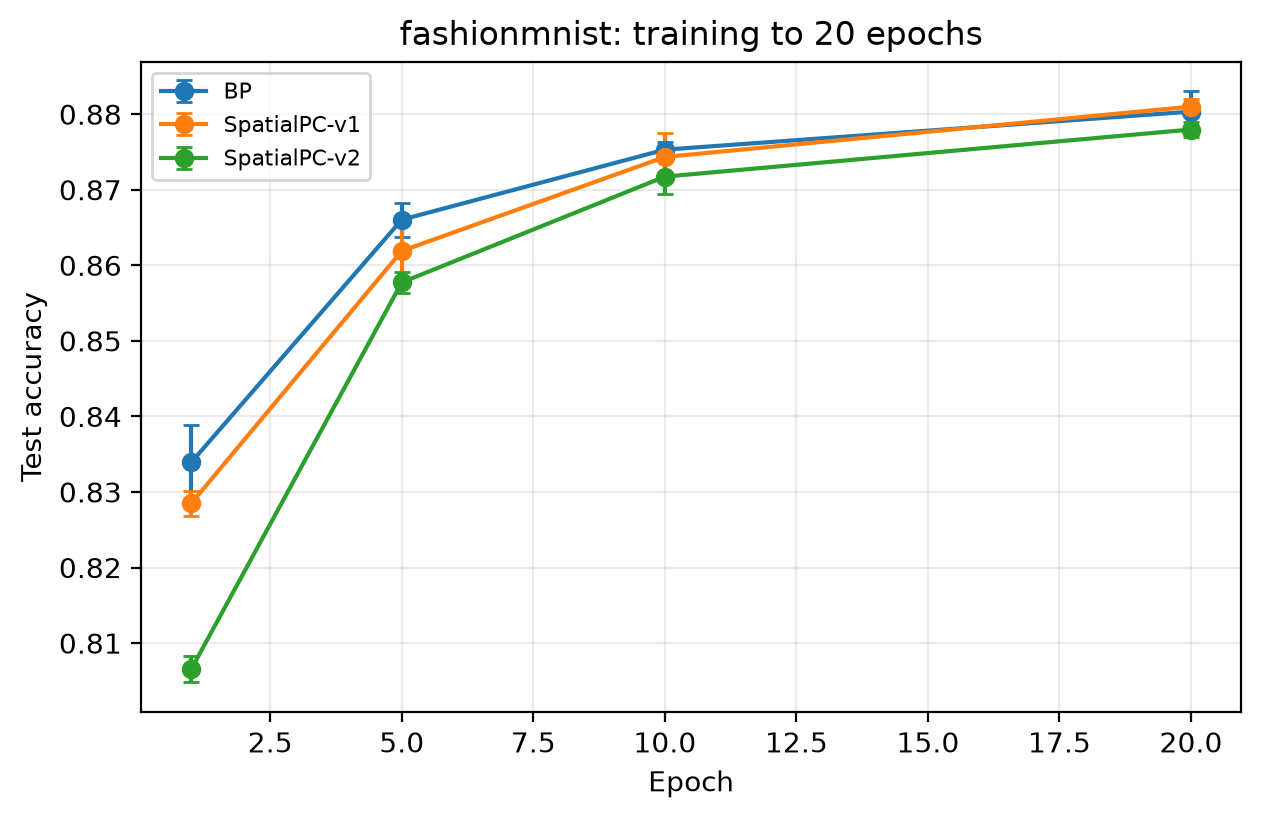}
        \caption{Fashion-MNIST accuracy}
    \end{subfigure}
    \begin{subfigure}[b]{0.49\linewidth}
        \centering
        \includegraphics[width=\linewidth]{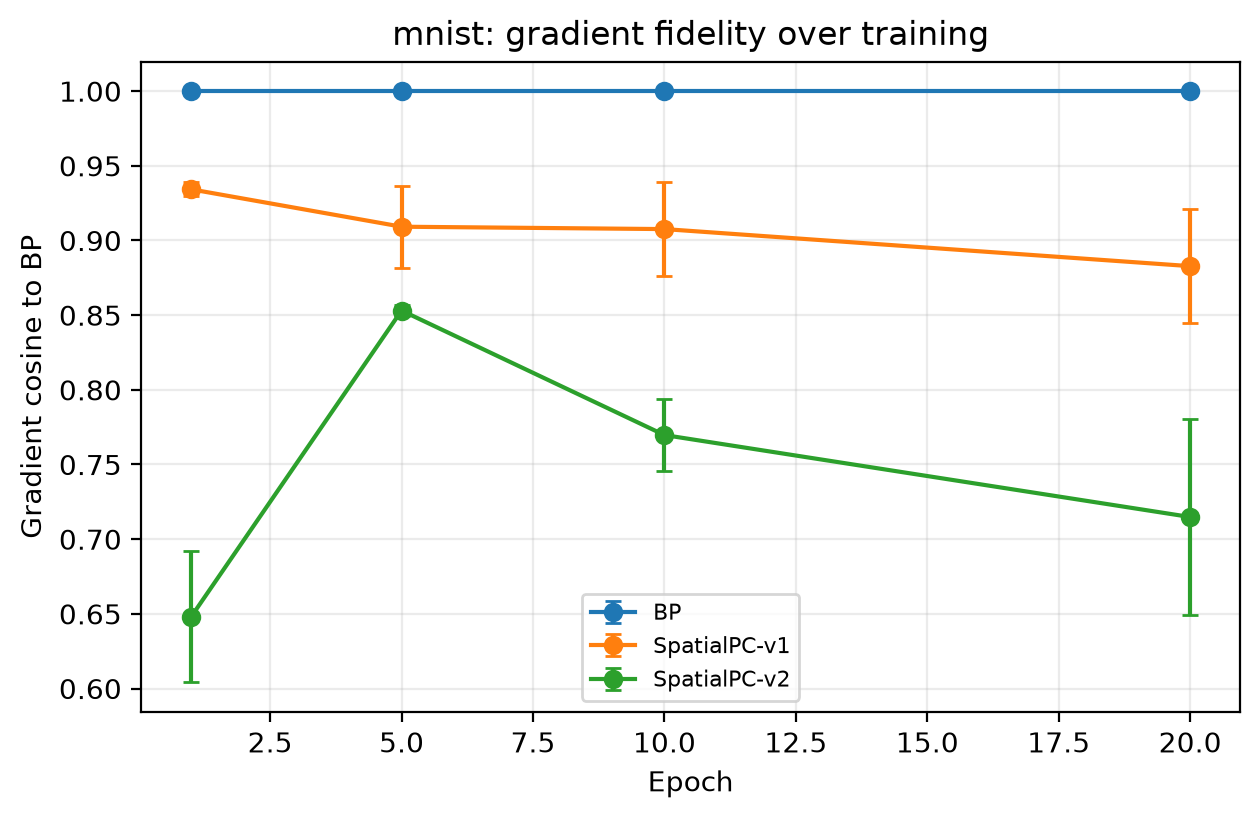}
        \caption{MNIST gradient cosine}
    \end{subfigure}\hfill
    \begin{subfigure}[b]{0.49\linewidth}
        \centering
        \includegraphics[width=\linewidth]{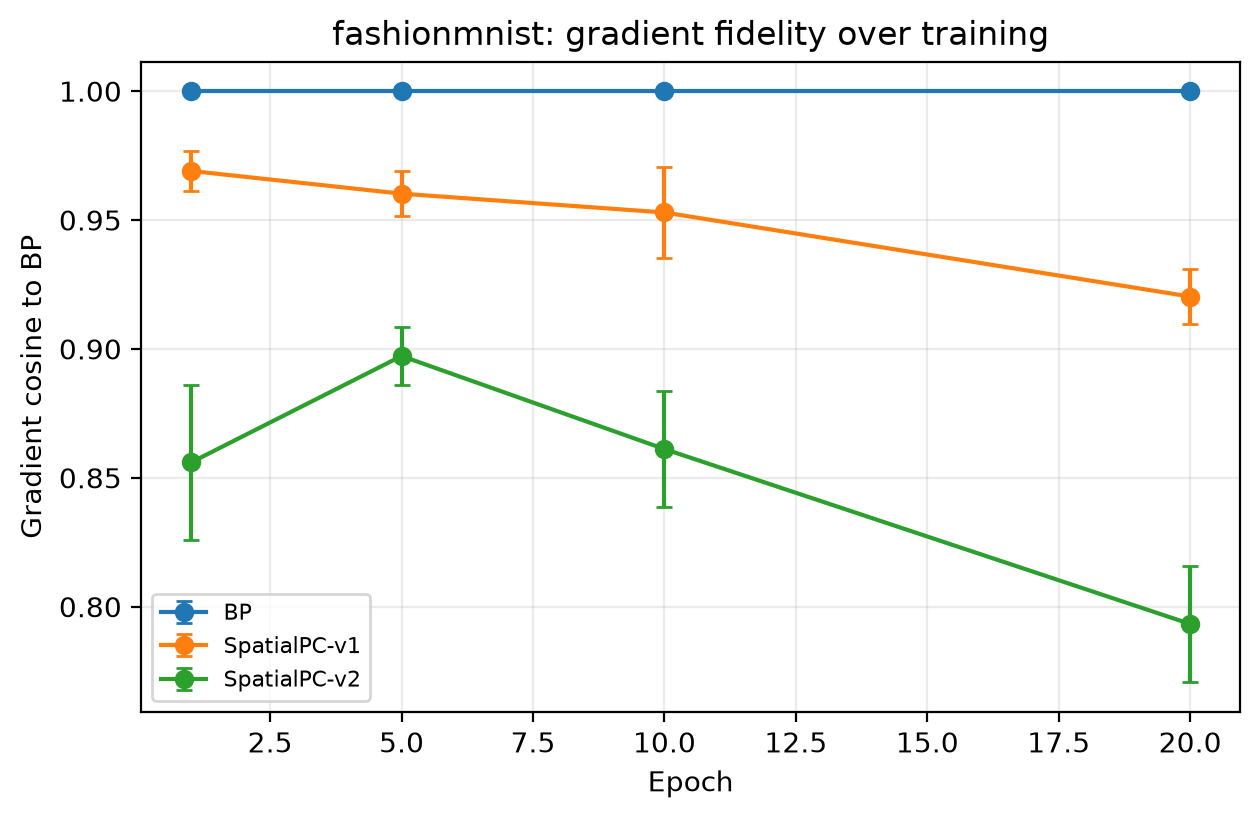}
        \caption{Fashion-MNIST gradient cosine}
    \end{subfigure}
    \caption{Finite-round SpatialPC over 20 epochs at $N=256$, $T=32$. Gradient cosine compares the finite-round APD task gradient with the settled gradient of the same Moore network. The dashed accuracy curve is the depth-128 backpropagation reference.}
    \label{fig:exp2}
\end{figure}

\begin{table}[htbp]
\centering
\small
\newcommand{\na}{\multicolumn{2}{c}{--}}
\caption{Twenty-epoch training at $N=256$, $T=32$, $H=128$ (mean $\pm$ std over three seeds). ``Test'' is finite-round accuracy from $h^T$ and ``Settled'' uses the same parameters after an exact sweep; RMS is the residual RMS. The BP row is a depth-128 chain reference, for which the SpatialPC diagnostics do not apply.}
\label{tab:exp2}
\begin{tabular}{@{} ll r *{4}{r@{\,{\scriptsize$\pm$}\,}>{\scriptsize}l} @{}} 
\toprule
 & & & \multicolumn{4}{c}{\textbf{Accuracy (\%)}} & \multicolumn{4}{c}{\textbf{Gradient Alignment}} \\
\cmidrule(lr){4-7} \cmidrule(l){8-11}
\textbf{Dataset} & \textbf{Method} & \textbf{Epoch}
  & \multicolumn{2}{c}{\textbf{Test} $\uparrow$}
  & \multicolumn{2}{c}{\textbf{Settled} $\uparrow$}
  & \multicolumn{2}{c}{\textbf{Task-grad cos.} $\uparrow$}
  & \multicolumn{2}{c}{\textbf{RMS} $\bm{(\times10^{-3})}$ $\downarrow$} \\
\midrule
MNIST  & BP    &  1 & 93.40 & 0.31 & \na          & \na           & \na \\
       &       &  5 & 96.54 & 0.15 & \na          & \na           & \na \\
       &       & 10 & 97.23 & 0.32 & \na          & \na           & \na \\
       &       & 20 & 97.55 & 0.09 & \na          & \na           & \na \\
\addlinespace
       & SPCv1 &  1 & 92.74 & 0.23 & 92.58 & 0.19 & 0.934 & 0.005 & 7.9 & 0.3 \\
       &       &  5 & 96.56 & 0.14 & 96.30 & 0.19 & 0.909 & 0.027 & 6.9 & 0.1 \\
       &       & 10 & 97.41 & 0.10 & 97.20 & 0.14 & 0.908 & 0.032 & 6.1 & 0.1 \\
       &       & 20 & 97.64 & 0.14 & 97.30 & 0.17 & 0.883 & 0.038 & 5.1 & 0.1 \\
\addlinespace
       & SPCv2 &  1 & 91.23 & 0.27 & 91.19 & 0.31 & 0.648 & 0.044 & 8.3 & 0.4 \\
       &       &  5 & 95.86 & 0.12 & 95.89 & 0.21 & 0.853 & 0.004 & 5.5 & 0.1 \\
       &       & 10 & 96.95 & 0.06 & 96.94 & 0.07 & 0.770 & 0.024 & 4.5 & 0.2 \\
       &       & 20 & 97.48 & 0.09 & 97.34 & 0.09 & 0.715 & 0.066 & 3.7 & 0.2 \\
\midrule
FMNIST & BP    &  1 & 83.39 & 0.49 & \na          & \na           & \na \\
       &       &  5 & 86.60 & 0.23 & \na          & \na           & \na \\
       &       & 10 & 87.53 & 0.10 & \na          & \na           & \na \\
       &       & 20 & 88.03 & 0.27 & \na          & \na           & \na \\
\addlinespace
       & SPCv1 &  1 & 82.85 & 0.17 & 82.59 & 0.23 & 0.969 & 0.008 & 5.9 & 0.2 \\
       &       &  5 & 86.19 & 0.47 & 85.87 & 0.73 & 0.960 & 0.009 & 4.9 & 0.1 \\
       &       & 10 & 87.43 & 0.32 & 87.19 & 0.48 & 0.953 & 0.018 & 4.8 & 0.1 \\
       &       & 20 & 88.10 & 0.11 & 87.36 & 0.06 & 0.920 & 0.011 & 5.0 & 0.0 \\
\addlinespace
       & SPCv2 &  1 & 80.66 & 0.18 & 80.65 & 0.12 & 0.856 & 0.030 & 5.3 & 0.2 \\
       &       &  5 & 85.77 & 0.14 & 85.74 & 0.04 & 0.897 & 0.011 & 5.0 & 0.2 \\
       &       & 10 & 87.17 & 0.22 & 86.83 & 0.33 & 0.861 & 0.022 & 4.3 & 0.1 \\
       &       & 20 & 87.80 & 0.09 & 87.73 & 0.03 & 0.793 & 0.023 & 3.8 & 0.1 \\
\bottomrule
\end{tabular}
\end{table}

\subsection{Chain control and external baselines}
\label{sec:controls}

Two further comparisons are in \cref{app:secondary-experiments}. A 256-state chain trained with the same $T=32$, width, objective, and optimizer stays at chance (\cref{tab:ablation-topology}): by \cref{prop:forward-causality}, no input information reaches the output of a chain in fewer than $256$ rounds, so the output state and the task gradient are identically zero, and the accuracies of $9.80\%$ and $10.00\%$ are the class-0 frequencies of the two test sets. The control is thus an instance of the forward light cone of \cref{prop:forward-causality}; it also changes the depth-dependent residual scaling of the implementation, so it is not an isolated topology intervention. A one-epoch width sweep against backpropagation, standard PC, $\mu$PC, and PC-ALM on depth-128 chains at matched width (\cref{tab:exp1}) shows SpatialPC-v1 ahead at $H\le32$ and within one point of backpropagation at $H=128$.

\section{Discussion}
\label{sec:discussion}

\noindent\textbf{Scope of the bounds.}
\Cref{thm:spatial-compiler,thm:light-cone} concern causal availability under bounded range, bounded degree, and constant density. Within that model, the $\Theta(N)$ latency of a chain is a consequence of embedding a long computation along a line, a $d$-dimensional embedding reduces it to $\Theta(N^{1/d})$ without shortening the computation, and no local algorithm does better. Outside the model the bounds do not apply: a reduction tree, a global broadcast, or any interaction whose range grows with $N$ breaks the light cone, and synthetic gradients \cite{jaderberg2017dni} sidestep it by predicting credit rather than transporting it. Within the model, the bounds say nothing about accuracy after arrival. By \cref{prop:finite-neumann} that is governed by $R_\epsilon(J,b)$, which depends on the learned weights and the task, and the experiments show the two scales to be far apart: at $N=256$ every state is reachable after 18 rounds, and gradient fidelity continues to improve to $T=64$.

\medskip\noindent\textbf{Circuit view.}
The $T$-round SpatialPC is itself a realizable circuit in the sense of \cref{def:realizable}: the spatial graph unrolled $T$ times in time, of width $N$ and depth $T$, with the same unit-length wiring at every layer. With $T=2\sqrt N$ it has $N=O(T^2)$ gates per layer inside a causal ball of radius $O(T)$, which is the volume constraint defining $\mathbf{RC}_2(T)$ \cite{prada2025realizable}; the geometric budget is thus, up to constants, the smallest round count at which a planar circuit of this size is realizable, and \cref{cor:optimal-exponent} shows that the same budget suffices for credit to reach every gate. By \cref{prop:forward-causality} the function it computes depends on the input only through paths of length at most $T-1$, so with $T=2\sqrt N$ it is a depth-$O(\sqrt N)$ circuit with unit-length wires, and the represented depth $N-1$ is a property of the wiring, the longest directed path available to the circuit, rather than of the $T$-round function. The light cone applies to the forward direction as well: an exact evaluation of the full backbone takes $N$ rounds on any local substrate, while the compiled graph delivers input information to every state within $d(m-1)$ rounds (\cref{thm:spatial-compiler}(ii)). \Cref{cor:optimal-exponent} therefore yields a circuit family in which the wiring supports arbitrarily long directed computations, both input and loss information reach every unit in $O(N^{1/d})$ rounds, and the function realized at budget $T$ is the $T$-round truncation, which the feasibility penalty trains to be self-consistent. Which paths a trained circuit uses is a question about the learned weights, taken up below.

\medskip\noindent\textbf{Geometry and dynamics.}
The scaling methods of \cref{sec:related} act on the dynamics that carry credit over a given graph and, in the terms of \cref{sec:finite-credit}, shorten the gap between $D_{\mathrm{reach}}$ and $R_\epsilon$. Spatial compilation reduces $D_{\mathrm{reach}}$. The two interventions are independent and compose: the compiler places no constraint on the local rule, and any of those dynamics can be run on a compiled graph. For relaxation-based models the path-tail analysis also supplies a diagnostic, since $R_\epsilon$ is then the number of inference iterations after which credit reaches a given relative accuracy, and comparing it with $D_{\mathrm{reach}}$ separates latency caused by distance from latency caused by the Jacobian tail.

\medskip\noindent\textbf{Physical implementation.}
SpatialPC is formulated for substrates in which range and degree are physical constraints, such as mesh-connected neuromorphic and in-memory accelerators. Every forward and reverse operation is local to an edge; the wiring is bounded-degree, unit-length, and fixed at compile time; and neither the compiler nor APD uses all-to-all communication, reduction trees, or a nonlocal feedback matrix. A computation can be laid out on such a fabric by ordering it along a space-filling traversal and retaining the unit-range shortcuts; a planar device cannot reduce worst-case credit latency below order $\sqrt N$ rounds, and a three-dimensional stack cannot reduce it below order $N^{1/3}$; and one parameter update costs $2T+2$ local message phases with $T$ on the order of $R_\epsilon$. The saving is in synchronization depth, not arithmetic: each round costs $O(|E|H^2)$ operations, so a $T$-round update performs $T$ times the work of one topological sweep, spread over $O(T)$ rounds instead of $O(N)$.

\medskip\noindent\textbf{Open problems.}
Four questions follow directly from the results. First, the Moore identity $D_{\mathrm{reach}}=\sqrt N+2$ holds at every implemented size but is not proved for all $k$; more generally, the minimum of $D_{\mathrm{reach}}$ over cycle-based compilers with $O(N)$ edges is not known, and \cref{thm:spatial-compiler} shows only that $2(\sqrt N-1)$ is attainable with $2N-2\sqrt N$ edges. Second, the tail identity reduces $R_\epsilon$ to the decay of $\|(J^\top)^Ka^\star\|$, which suggests that the depth-aware parameterizations of \cite{innocenti2025mupc,innocenti2026limits} could be used to bound $R_\epsilon$ on the compiled graph in terms of $N$; the present experiments measure $R_\epsilon$ through gradient alignment at $N\le256$ and leave its growth with $N$ open. Third, represented depth is a property of the wiring; a path-sensitive ablation that cuts long backbone segments after training would measure how much of it a trained circuit uses. Fourth, the composition of geometry with relaxation-based dynamics, which the analysis predicts, is a direct experiment: run PC-ALM or $\mu$PC on the Moore graph and on the chain at matched $N$ and compare their inference-iteration counts. Three-dimensional substrates, convolutional and sequence computations, and measurements on physically local hardware extend the same program.

\section{Conclusion}
\label{sec:conclusion}

The communication latency of local credit assignment is a property of the graph on which credit travels, not of the depth of the computation. For $N=m^d$ states on a constant-density, bounded-range substrate, $\Omega(N^{1/d})$ rounds are necessary for loss information to reach every loss-sensitive parameter, and a unit-range, degree-$2d$ compiler reaches every state in $d(N^{1/d}-1)$ rounds while retaining an $N-1$ Hamiltonian computation; the separation is tight for every local learning algorithm. SpatialPC realizes it in predictive coding with an exact, edge-local gradient, and its path-length expansion identifies the second quantity on which finite-round learning depends: the effective credit radius, set by the decay of Jacobian powers along the settled adjoint rather than by the graph. The sparse Moore compiler keeps the two-dimensional exponent with $\tfrac32N-2$ edges to $65{,}536$ states, and under $O(\sqrt N)$ round budgets the trained models keep finite-round gradients aligned with settled gradients and match a depth-128 backpropagation reference on MNIST and Fashion-MNIST. The compiler is defined for every dimension, the analysis applies to any local rule, and together they make the communication cost of credit assignment a quantity that can be chosen by design rather than inherited from depth.

\bibliographystyle{unsrtnat}
\bibliography{references}

@inproceedings{ororbia2019biologically,
  title     = {Biologically Motivated Algorithms for Propagating Local Target Representations},
  author    = {Ororbia, Alexander G. and Mali, Ankur},
  booktitle = {Proceedings of the AAAI Conference on Artificial Intelligence},
  volume    = {33},
  number    = {01},
  pages     = {4651--4658},
  year      = {2019},
  doi       = {10.1609/aaai.v33i01.33014651}
}

@article{salvatori2025survey,
  title   = {A Survey on Neuro-Mimetic Deep Learning via Predictive Coding},
  author  = {Salvatori, Tommaso and Mali, Ankur and Buckley, Christopher L. and Lukasiewicz, Thomas and Rao, Rajesh P. N. and Friston, Karl and Ororbia, Alexander},
  journal = {Neural Networks},
  volume  = {195},
  pages   = {108161},
  year    = {2026},
  month   = mar,
  doi     = {10.1016/j.neunet.2025.108161}
}

@misc{prada2025realizable,
  title         = {Realizable Circuit Complexity: Embedding Computation in Space--Time},
  author        = {Prada, Benjamin and Mali, Ankur},
  year          = {2025},
  eprint        = {2509.19161},
  archivePrefix = {arXiv},
  primaryClass  = {cs.CC},
  url           = {https://arxiv.org/abs/2509.19161}
}

@inproceedings{lee2015difference,
  title     = {Difference Target Propagation},
  author    = {Lee, Dong-Hyun and Zhang, Saizheng and Fischer, Asja and Bengio, Yoshua},
  booktitle = {Machine Learning and Knowledge Discovery in Databases},
  pages     = {498--515},
  year      = {2015},
  publisher = {Springer}
}

@misc{ororbia2024review,
  title         = {A Review of Neuroscience-Inspired Machine Learning},
  author        = {Ororbia, Alexander and Mali, Ankur and Kohan, Adam and Millidge, Beren and Salvatori, Tommaso},
  year          = {2024},
  eprint        = {2403.18929},
  archivePrefix = {arXiv},
  primaryClass  = {cs.LG},
  url           = {https://arxiv.org/abs/2403.18929}
}

@misc{naumov2017parallel,
  title         = {Parallel Complexity of Forward and Backward Propagation},
  author        = {Naumov, Maxim},
  year          = {2017},
  eprint        = {1712.06577},
  archivePrefix = {arXiv},
  primaryClass  = {cs.LG},
  url           = {https://arxiv.org/abs/1712.06577}
}

@inproceedings{wang2020bppsa,
  title     = {{BPPSA}: Scaling Back-Propagation by Parallel Scan Algorithm},
  author    = {Wang, Shang and Bai, Yifan and Pekhimenko, Gennady},
  booktitle = {Proceedings of Machine Learning and Systems},
  volume    = {2},
  year      = {2020},
  publisher = {MLSys}
}

@misc{song2021accelerating,
  title         = {Nonlinear Equation Solving: A Faster Alternative to Feedforward Computation},
  author        = {Song, Yang and Meng, Chenlin and Liao, Renjie and Ermon, Stefano},
  year          = {2020},
  eprint        = {2002.03629},
  archivePrefix = {arXiv},
  primaryClass  = {cs.LG},
  doi           = {10.48550/arXiv.2002.03629},
  url           = {https://arxiv.org/abs/2002.03629}
}

@inproceedings{jaderberg2017dni,
  title     = {Decoupled Neural Interfaces using Synthetic Gradients},
  author    = {Jaderberg, Max and Czarnecki, Wojciech Marian and Osindero, Simon and Vinyals, Oriol and Graves, Alex and Silver, David and Kavukcuoglu, Koray},
  booktitle = {Proceedings of the 34th International Conference on Machine Learning},
  pages     = {1627--1635},
  year      = {2017},
  volume    = {70},
  series    = {Proceedings of Machine Learning Research},
  publisher = {PMLR},
  url       = {https://proceedings.mlr.press/v70/jaderberg17a.html}
}

@article{millidge2020predictive,
  title   = {Predictive Coding Approximates Backprop Along Arbitrary Computation Graphs},
  author  = {Millidge, Beren and Tschantz, Alexander and Buckley, Christopher L.},
  journal = {Neural Computation},
  volume  = {34},
  number  = {6},
  pages   = {1329--1368},
  year    = {2022},
  doi     = {10.1162/neco_a_01497}
}

@inproceedings{salvatori2022graphs,
  title     = {Learning on Arbitrary Graph Topologies via Predictive Coding},
  author    = {Salvatori, Tommaso and Pinchetti, Luca and Millidge, Beren and Song, Yuhang and Bao, Tianyi and Bogacz, Rafal and Lukasiewicz, Thomas},
  booktitle = {Advances in Neural Information Processing Systems},
  volume    = {35},
  year      = {2022},
  publisher = {Curran Associates, Inc.},
  doi       = {10.52202/068431-2770}
}

@article{scellier2017equilibrium,
  title   = {Equilibrium Propagation: Bridging the Gap between Energy-Based Models and Backpropagation},
  author  = {Scellier, Benjamin and Bengio, Yoshua},
  journal = {Frontiers in Computational Neuroscience},
  volume  = {11},
  pages   = {24},
  year    = {2017},
  doi     = {10.3389/fncom.2017.00024}
}

@misc{seely2026pcalm,
  title         = {Augmented Lagrangian Predictive Coding},
  author        = {Seely, Jeffrey and Gould, Julian},
  year          = {2026},
  eprint        = {2605.31022},
  archivePrefix = {arXiv},
  primaryClass  = {cs.LG},
  url           = {https://arxiv.org/abs/2605.31022}
}

@article{Moore1900,
  author  = {Moore, Eliakim Hastings},
  title   = {On Certain Crinkly Curves},
  journal = {Transactions of the American Mathematical Society},
  volume  = {1},
  number  = {1},
  pages   = {72--90},
  year    = {1900},
  doi     = {10.2307/1986405}
}

@article{Rumelhart1986,
  author  = {Rumelhart, David E. and Hinton, Geoffrey E. and Williams, Ronald J.},
  title   = {Learning Representations by Back-Propagating Errors},
  journal = {Nature},
  volume  = {323},
  number  = {6088},
  pages   = {533--536},
  year    = {1986},
  doi     = {10.1038/323533a0}
}

@article{Lillicrap2020,
  author  = {Lillicrap, Timothy P. and Santoro, Adam and Marris, Luke and Akerman, Colin J. and Hinton, Geoffrey},
  title   = {Backpropagation and the Brain},
  journal = {Nature Reviews Neuroscience},
  volume  = {21},
  number  = {6},
  pages   = {335--346},
  year    = {2020},
  doi     = {10.1038/s41583-020-0277-3}
}

@article{Hestenes1969,
  author  = {Hestenes, Magnus R.},
  title   = {Multiplier and Gradient Methods},
  journal = {Journal of Optimization Theory and Applications},
  volume  = {4},
  number  = {5},
  pages   = {303--320},
  year    = {1969},
  doi     = {10.1007/BF00927673}
}

@misc{google-s2,
  author       = {{Google}},
  title        = {{S2 Geometry}},
  howpublished = {\url{https://s2geometry.io/}},
  url          = {https://s2geometry.io/},
  note         = {Software documentation}
}

@misc{mnist,
  author       = {LeCun, Yann and Cortes, Corinna and Burges, Christopher J. C.},
  title        = {The {MNIST} Database of Handwritten Digits},
  howpublished = {\url{http://yann.lecun.com/exdb/mnist/}},
  url          = {http://yann.lecun.com/exdb/mnist/}
}

@article{fashionmnist,
  author        = {Xiao, Han and Rasul, Kashif and Vollgraf, Roland},
  title         = {Fashion-{MNIST}: A Novel Image Dataset for Benchmarking Machine Learning Algorithms},
  journal       = {CoRR},
  volume        = {abs/1708.07747},
  year          = {2017},
  eprint        = {1708.07747},
  archivePrefix = {arXiv},
  url           = {https://arxiv.org/abs/1708.07747}
}

@inproceedings{innocenti2025mupc,
  title     = {{$\mu$PC}: Scaling Predictive Coding to 100+ Layer Networks},
  author    = {Innocenti, Francesco and Achour, El Mehdi and Buckley, Christopher L.},
  booktitle = {Advances in Neural Information Processing Systems},
  volume    = {38},
  year      = {2025},
  publisher = {Curran Associates, Inc.},
  doi       = {10.52202/085713-2923}
}

@inproceedings{kingma2017adam,
  title     = {Adam: A Method for Stochastic Optimization},
  author    = {Kingma, Diederik P. and Ba, Jimmy},
  booktitle = {International Conference on Learning Representations},
  year      = {2015},
  eprint    = {1412.6980},
  archivePrefix = {arXiv},
  url       = {https://arxiv.org/abs/1412.6980}
}

@inproceedings{paszke2019pytorch,
  title     = {{PyTorch}: An Imperative Style, High-Performance Deep Learning Library},
  author    = {Paszke, Adam and Gross, Sam and Massa, Francisco and Lerer, Adam and Bradbury, James and Chanan, Gregory and Killeen, Trevor and Lin, Zeming and Gimelshein, Natalia and Antiga, Luca and Desmaison, Alban and K{\"o}pf, Andreas and Yang, Edward and DeVito, Zachary and Raison, Martin and Tejani, Alykhan and Chilamkurthy, Sasank and Steiner, Benoit and Fang, Lu and Bai, Junjie and Chintala, Soumith},
  booktitle = {Advances in Neural Information Processing Systems},
  volume    = {32},
  year      = {2019}
}

@inproceedings{pmlr-v80-wen18a,
  title     = {Deep Predictive Coding Network for Object Recognition},
  author    = {Wen, Haiguang and Han, Kuan and Shi, Junxing and Zhang, Yizhen and Culurciello, Eugenio and Liu, Zhongming},
  booktitle = {Proceedings of the 35th International Conference on Machine Learning},
  pages     = {5266--5275},
  year      = {2018},
  volume    = {80},
  series    = {Proceedings of Machine Learning Research},
  publisher = {PMLR},
  url       = {https://proceedings.mlr.press/v80/wen18a.html}
}

@misc{millidge2022predictivecodingtheoreticalexperimental,
  title         = {Predictive Coding: A Theoretical and Experimental Review},
  author        = {Millidge, Beren and Seth, Anil and Buckley, Christopher L.},
  year          = {2021},
  eprint        = {2107.12979},
  archivePrefix = {arXiv},
  primaryClass  = {cs.AI},
  url           = {https://arxiv.org/abs/2107.12979}
}

@book{Prusinkiewicz1996,
  author    = {Prusinkiewicz, Przemyslaw and Lindenmayer, Aristid},
  title     = {The Algorithmic Beauty of Plants},
  publisher = {Springer-Verlag},
  year      = {1990},
  isbn      = {978-0-387-97297-8}
}

@misc{hinton2022forwardforwardalgorithmpreliminaryinvestigations,
  title         = {The Forward-Forward Algorithm: Some Preliminary Investigations},
  author        = {Hinton, Geoffrey},
  year          = {2022},
  eprint        = {2212.13345},
  archivePrefix = {arXiv},
  primaryClass  = {cs.LG},
  url           = {https://arxiv.org/abs/2212.13345}
}

@inproceedings{salvatori2024ipc,
  title     = {A Stable, Fast, and Fully Automatic Learning Algorithm for Predictive Coding Networks},
  author    = {Salvatori, Tommaso and Song, Yuhang and Yordanov, Yordan and Millidge, Beren and Sha, Lei and Emde, Cornelius and Xu, Zhenghua and Bogacz, Rafal and Lukasiewicz, Thomas},
  booktitle = {International Conference on Learning Representations},
  year      = {2024}
}

@inproceedings{salvatori2022reverse,
  title     = {Reverse Differentiation via Predictive Coding},
  author    = {Salvatori, Tommaso and Song, Yuhang and Xu, Zhenghua and Lukasiewicz, Thomas and Bogacz, Rafal},
  booktitle = {Proceedings of the AAAI Conference on Artificial Intelligence},
  volume    = {36},
  number    = {7},
  pages     = {8150--8158},
  year      = {2022},
  doi       = {10.1609/aaai.v36i7.20788}
}

@inproceedings{qi2026deeper,
  title     = {Towards the Training of Deeper Predictive Coding Neural Networks},
  author    = {Qi, Chang and Forasassi, Matteo and Lukasiewicz, Thomas and Salvatori, Tommaso},
  booktitle = {Proceedings of the 43rd International Conference on Machine Learning},
  series    = {Proceedings of Machine Learning Research},
  volume    = {306},
  pages     = {100334--100362},
  year      = {2026},
  publisher = {PMLR},
  url       = {https://proceedings.mlr.press/v306/qi26f.html}
}

@misc{mohammadi2026hep,
  title         = {Error Highways: Scaling Predictive Coding to Very Deep Networks},
  author        = {Mohammadi, Amirhossein and Ororbia, Alexander G.},
  year          = {2026},
  eprint        = {2606.22744},
  archivePrefix = {arXiv},
  primaryClass  = {cs.LG},
  url           = {https://arxiv.org/abs/2606.22744}
}

@inproceedings{goemaere2026epc,
  title     = {e{PC}: Fast and Deep Predictive Coding in Digital Simulation},
  author    = {Goemaere, C{\'e}dric and Oliviers, Gaspard and Bogacz, Rafal and Demeester, Thomas},
  booktitle = {Proceedings of the 43rd International Conference on Machine Learning},
  series    = {Proceedings of Machine Learning Research},
  volume    = {306},
  pages     = {35407--35442},
  year      = {2026},
  publisher = {PMLR},
  url       = {https://proceedings.mlr.press/v306/goemaere26a.html}
}

@inproceedings{innocenti2026limits,
  title     = {On the Infinite Width and Depth Limits of Predictive Coding Networks},
  author    = {Innocenti, Francesco and Achour, El Mehdi and Bogacz, Rafal},
  booktitle = {Proceedings of the 43rd International Conference on Machine Learning},
  series    = {Proceedings of Machine Learning Research},
  volume    = {306},
  pages     = {49761--49799},
  year      = {2026},
  publisher = {PMLR},
  url       = {https://proceedings.mlr.press/v306/innocenti26a.html}
}

@inproceedings{ha2026metapcn,
  title     = {Stable and Scalable Deep Predictive Coding Networks with Meta-Prediction Errors},
  author    = {Ha, Myoung Hoon and Kim, Hyunjun and Sung, Yoondo and Jo, Youngha and Kang, Min S. and Lee, Sang Wan},
  booktitle = {International Conference on Learning Representations},
  year      = {2026},
  url       = {https://openreview.net/forum?id=kE5jJUHl9i}
}

@misc{kerjan2026imagenet,
  title         = {Training a Predictive Coding Network on {ImageNet} using Equilibrium Propagation},
  author        = {Kerjan, Tugdual and H{\o}ier, Rasmus and Scellier, Benjamin},
  year          = {2026},
  eprint        = {2606.03584},
  archivePrefix = {arXiv},
  primaryClass  = {cs.LG},
  url           = {https://arxiv.org/abs/2606.03584}
}

@article{linial1992locality,
  author  = {Linial, Nathan},
  title   = {Locality in Distributed Graph Algorithms},
  journal = {SIAM Journal on Computing},
  volume  = {21},
  number  = {1},
  pages   = {193--201},
  year    = {1992},
  doi     = {10.1137/0221015}
}

@book{peleg2000distributed,
  author    = {Peleg, David},
  title     = {Distributed Computing: A Locality-Sensitive Approach},
  publisher = {SIAM},
  series    = {SIAM Monographs on Discrete Mathematics and Applications},
  year      = {2000}
}

@inproceedings{bai2019deq,
  author    = {Bai, Shaojie and Kolter, J. Zico and Koltun, Vladlen},
  title     = {Deep Equilibrium Models},
  booktitle = {Advances in Neural Information Processing Systems},
  volume    = {32},
  year      = {2019}
}

@article{vitanyi1988locality,
  author  = {Vit{\'a}nyi, Paul M. B.},
  title   = {Locality, Communication, and Interconnect Length in Multicomputers},
  journal = {SIAM Journal on Computing},
  volume  = {17},
  number  = {4},
  pages   = {659--672},
  year    = {1988},
  doi     = {10.1137/0217042}
}

@phdthesis{thompson1980vlsi,
  author = {Thompson, Clark D.},
  title  = {A Complexity Theory for {VLSI}},
  school = {Carnegie Mellon University},
  year   = {1980},
  note   = {Technical Report CMU-CS-80-140}
}

@book{varela1991embodied,
  author    = {Varela, Francisco J. and Thompson, Evan and Rosch, Eleanor},
  title     = {The Embodied Mind: Cognitive Science and Human Experience},
  publisher = {MIT Press},
  address   = {Cambridge, MA},
  year      = {1991}
}

@book{pfeifer2006body,
  author    = {Pfeifer, Rolf and Bongard, Josh},
  title     = {How the Body Shapes the Way We Think: A New View of Intelligence},
  publisher = {MIT Press},
  address   = {Cambridge, MA},
  year      = {2006}
}

\appendix
\crefalias{section}{appendix}
\crefalias{subsection}{appendix}
\clearpage
\section*{Appendix}

\noindent The appendix provides detailed proofs and protocols for materials referenced in the main text. \Cref{app:proofs-compilation} proves the monotone-edge lemma behind \cref{thm:spatial-compiler} and records how the Moore reachability of \cref{eq:moore-reach-empirical} was measured; \cref{app:training-derivation} derives the finite-round APD recursion of \cref{prop:finite-spatialpc-apd}; \cref{app:angle-proof} proves \cref{cor:gradient-angle}; \cref{app:apd-pcalm} relates the adjoint equation to augmented-Lagrangian predictive coding; \cref{app:gradient-verification} verifies the implementation against automatic differentiation; \cref{app:secondary-experiments} reports the chain control, the activation ablation, and the external width sweep; and \cref{app:algorithm} gives pseudocode for the compiler, the forward pass, and training.

\section{Proofs for the Spatial Compilation Results}
\label{app:proofs-compilation}

\subsection{Proof of \texorpdfstring{\cref{lem:monotone}}{the monotone-edge lemma}}
\Cref{lem:monotone} is the step of \cref{thm:spatial-compiler} that supplies, from every point of the snake-ordered grid, a forward edge that moves one unit closer to the output. Write $v_{\mathrm{in}}^{(d)}$ and $v_{\mathrm{out}}^{(d)}$ for the first and last points of $\pi_{m,d}$, and call the two statements of the lemma (A$_d$) and (B$_d$):
\begin{itemize}[itemsep=0pt,topsep=1pt]
    \item[(A$_d$)] every $v\ne v_{\mathrm{out}}^{(d)}$ has a unit neighbor $u$ with
    $$\pi_{m,d}(u)>\pi_{m,d}(v) \text{ and } \|u-v_{\mathrm{out}}^{(d)}\|_1=\|v-v_{\mathrm{out}}^{(d)}\|_1-1;$$
    \item[(B$_d$)] every $v\ne v_{\mathrm{in}}^{(d)}$ has a unit neighbor $u$ with
    $$\pi_{m,d}(u)<\pi_{m,d}(v) \text{ and } \|u-v_{\mathrm{in}}^{(d)}\|_1=\|v-v_{\mathrm{in}}^{(d)}\|_1-1.$$
\end{itemize}
We prove (A$_d$) and (B$_d$) together by induction on $d$.

\emph{Base case $d=1$.} Here $v_{\mathrm{in}}^{(1)}=0$ and $v_{\mathrm{out}}^{(1)}=m-1$. For $v<m-1$ take $u=v+1$; for $v>0$ take $u=v-1$.

\emph{Inductive step.} Assume (A$_{d-1}$) and (B$_{d-1}$). By \cref{eq:snake-order}, $v_{\mathrm{in}}^{(d)}=(0,v_{\mathrm{in}}^{(d-1)})$, and $v_{\mathrm{out}}^{(d)}=(m-1,v_{\mathrm{out}}^{(d-1)})$ if $m-1$ is even and $(m-1,v_{\mathrm{in}}^{(d-1)})$ if $m-1$ is odd. Points in distinct slabs are ordered by slab index, and within slab $i$ the order is $\pi_{m,d-1}$ if $i$ is even and its reversal if $i$ is odd.

(A$_d$): let $v=(i,v')\neq v_{\mathrm{out}}^{(d)}$. If $i<m-1$, take $u=(i+1,v')$. It lies in a later slab, so $\pi_{m,d}(u)>\pi_{m,d}(v)$, and since the first coordinate of $v_{\mathrm{out}}^{(d)}$ is $m-1>i$, the move reduces the $\ell_1$ distance to $v_{\mathrm{out}}^{(d)}$ by one. If $i=m-1$ and $m-1$ is even, the last slab is traversed by $\pi_{m,d-1}$, whose last point is $v_{\mathrm{out}}^{(d-1)}\neq v'$; by (A$_{d-1}$) there is a unit neighbor $u'$ of $v'$ with $\pi_{m,d-1}(u')>\pi_{m,d-1}(v')$ one unit closer to $v_{\mathrm{out}}^{(d-1)}$, and $u=(m-1,u')$ satisfies the claim. If $i=m-1$ and $m-1$ is odd, the last slab is traversed by the reversal of $\pi_{m,d-1}$, whose last point is $v_{\mathrm{in}}^{(d-1)}\neq v'$; by (B$_{d-1}$) there is a unit neighbor $u'$ with $\pi_{m,d-1}(u')<\pi_{m,d-1}(v')$, hence later in the reversed order, one unit closer to $v_{\mathrm{in}}^{(d-1)}$, and $u=(m-1,u')$ satisfies the claim.

(B$_d$): let $v=(i,v')\neq v_{\mathrm{in}}^{(d)}$. If $i>0$, take $u=(i-1,v')$, which lies in an earlier slab and is one unit closer to $v_{\mathrm{in}}^{(d)}$, whose first coordinate is $0$. If $i=0$, slab $0$ is traversed by $\pi_{m,d-1}$ with first point $v_{\mathrm{in}}^{(d-1)}\ne v'$, and (B$_{d-1}$) supplies $u'$; take $u=(0,u')$. \qed

\subsection{Reverse eccentricity of the Moore output}
\Cref{eq:moore-reach-empirical} was obtained by breadth-first search on the reversed graph from the output node for $m=2^k$, $k=2,\dots,8$; at $m=4$ the farthest state is state $1$ (\cref{fig:moore-grid}), and at every implemented size the maximum distance is exactly $m+2$. A proof for all $k$ would need to track the position of the cut (backbone index $0$ is placed one quarter of the way around the Moore cycle, \cref{alg:moore-dag}) relative to the recursive structure of the curve.

\section{Derivation of Finite-Round APD}
\label{app:training-derivation}
This appendix carries out the stationarity computation stated in the proof of \cref{prop:finite-spatialpc-apd}. For one example, define
\begin{equation}
    \mathcal J_T
    =\ell(\hat y,y)+\frac{\lambda}{2NH}\|r^T\|_2^2,
    \qquad
    r^T=h^T-F_\theta(h^T;x),
    \qquad
    \hat y=\readout(h^T),
\end{equation}
subject to the time-indexed constraints
\begin{equation}
    h^{t+1}-F_\theta(h^t;x)=0,
    \qquad t=0,\ldots,T-1 .
\end{equation}
With the sign convention of \cref{eq:pc-lagrangian},
\begin{equation}
    \mathcal L_T
    =\mathcal J_T
     -\sum_{t=0}^{T-1}(a^{t+1})^\top
        \bigl(h^{t+1}-F_\theta(h^t;x)\bigr).
    \label{eq:unrolled-lagrangian}
\end{equation}
Stationarity in $h^T$ gives $a^T=\nabla_{h^T}\mathcal J_T$. The feasibility contribution is
\begin{align}
\nabla_{h^T}\frac{\lambda}{2NH}\|r^T\|_2^2
&=\frac{\lambda}{NH}
  \left(\frac{\partial r^T}{\partial h^T}\right)^{\top} r^T
=\frac{\lambda}{NH}(I-J_T^\top)\,r^T,
\end{align}
so
\begin{equation}
    a^T
    =\nabla_{h^T}\ell
     +\frac{\lambda}{NH}\bigl(r^T-J_T^\top r^T\bigr),
\end{equation}
which is \cref{eq:terminal-apd}. For $t<T$, stationarity in $h^t$ gives $a^t=J_t^\top a^{t+1}$, which is \cref{eq:finite-apd-recursion}. Differentiating $\mathcal L_T$ in $\theta$ at a stationary point,
\begin{equation}
    \nabla_\theta\mathcal J_T
    =\left.\nabla_\theta\mathcal J_T\right|_{h^T\ \mathrm{fixed}}
     +\sum_{t=0}^{T-1}
      \left(\frac{\partial F_\theta(h^t;x)}{\partial\theta}\right)^{\top}a^{t+1},
\end{equation}
where the direct derivative of the feasibility term, holding $h^T$ fixed, is
\begin{equation}
\left.\nabla_\theta\frac{\lambda}{2NH}\|r^T\|_2^2\right|_{h^T\ \mathrm{fixed}}
=-\frac{\lambda}{NH}
  \left(\frac{\partial F_\theta(h^T;x)}{\partial\theta}\right)^{\top} r^T .
\end{equation}
This is the gradient evaluated by \cref{alg:spatialpc-training}.

\section{Proof of the Gradient-Angle Bound}
\label{app:angle-proof}
We prove \cref{cor:gradient-angle}. Let $x=a_v^\star$ and $y=a_v^K=x+\delta$ with $\|\delta\|\le\epsilon\|x\|$ and $0\le\epsilon<1$. Normalize $\|x\|=1$ and decompose
\begin{equation}
    \delta=\alpha x+\delta_\perp,
    \qquad x^\top\delta_\perp=0,
    \qquad \alpha^2+\|\delta_\perp\|^2\le\epsilon^2 .
\end{equation}
Then
\begin{equation}
\cos\angle(x,y)
=\frac{1+\alpha}{\sqrt{(1+\alpha)^2+\|\delta_\perp\|^2}},
\end{equation}
which is decreasing in $\|\delta_\perp\|^2$, so for fixed $\alpha\in[-\epsilon,\epsilon]$ the minimum is attained at $\|\delta_\perp\|^2=\epsilon^2-\alpha^2$:
\begin{equation}
\cos^2\angle(x,y)
\ \ge\
f(\alpha):=\frac{(1+\alpha)^2}{1+2\alpha+\epsilon^2}.
\end{equation}
Since $f'(\alpha)=2(1+\alpha)(\alpha+\epsilon^2)/(1+2\alpha+\epsilon^2)^2$ and $1+\alpha>0$, $f$ is minimized at $\alpha=-\epsilon^2$, where $f(-\epsilon^2)=(1-\epsilon^2)^2/(1-\epsilon^2)=1-\epsilon^2$. Hence $\cos\angle(x,y)\ge\sqrt{1-\epsilon^2}$, with equality when $y$ is tangent to the ball of radius $\epsilon$ about $x$. For the edge gradients $g^\star=\kappa xz^\top$ and $g^K=\kappa yz^\top$ with a common $z\neq0$, the Frobenius inner products factor as $\langle g^K,g^\star\rangle_F=\kappa^2\|z\|^2\langle y,x\rangle$, so $\cos\angle(g^K,g^\star)=\cos\angle(y,x)$. \qed

\section{Connection to Augmented-Lagrangian Predictive Coding}
\label{app:apd-pcalm}
\Cref{sec:apd} states that the adjoint equation \cref{eq:settled-adjoint} is the state-stationarity condition of augmented-Lagrangian predictive coding \cite{seely2026pcalm}. For the augmented Lagrangian \cite{Hestenes1969}
\begin{equation}
    \mathcal L_\beta(h,\theta,\zeta)
    =\ell(h_O,y)+\zeta^\top r(h,\theta)
      +\frac{\beta}{2}\|r(h,\theta)\|^2,
\end{equation}
define the augmented multiplier $\eta=\zeta+\beta r$ and set $a=-\eta$. Since $r=h-F_\theta(h;x)$,
\begin{equation}
    \nabla_h\mathcal L_\beta
    =\nabla_h\ell+(I-J^\top)\eta
    =\nabla_h\ell-(I-J^\top)a .
\end{equation}
State stationarity is therefore $a=\nabla_h\ell+J^\top a$, the settled APD equation \cref{eq:settled-adjoint}, at every value of the residual. At feasibility $\eta=\zeta$, so the APD adjoint coincides with minus the Lagrange multiplier of the constrained problem \cref{eq:constrained-pc}. SpatialPC shares this fixed-point condition with augmented-Lagrangian PC but not its dynamics: PC-ALM approaches the condition by primal--dual relaxation of states and multipliers, whereas APD on the unrolled graph (\cref{sec:finite-apd}) computes the exact gradient of the finite-round objective in $T$ reverse rounds with no relaxation and no feasibility requirement. The geometric results of \cref{sec:spatialpc} apply to both.

\section{Numerical Verification of the Implementation}
\label{app:gradient-verification}
\Cref{prop:finite-spatialpc-apd} states that APD on the unrolled graph computes the exact gradient of \cref{eq:spatial-objective}; this appendix checks the implementation of that computation. The hand-written finite-round APD implementation was compared against PyTorch autograd \cite{paszke2019pytorch} in double precision for both task losses, both readout scalings, and ReLU and tanh activations. Across the eight combinations the gradient cosine is $1.0$ to displayed precision and the relative error is below $2.9\times10^{-16}$ (\cref{tab:gradient-verification}). Stationary APD at convergence reproduces the settled backpropagation gradient for both the cross-entropy/$H^{-1/2}$ and MSE/$H^{-1}$ protocols with relative error below $3.4\times10^{-16}$.

\begin{table}[htbp]
\centering
\caption{Double-precision verification of the finite-round SpatialPC/APD gradient against automatic differentiation.}
\label{tab:gradient-verification}
\small
\begin{tabular}{llllr}
\toprule
\textbf{Loss} & \textbf{Readout} & \textbf{Activation} & \textbf{Gradient cosine} $\uparrow$ & \textbf{Relative error} $\downarrow$ \\
\midrule
CE  & $H^{-1/2}$ & ReLU & 1.0000000000 & $2.83\times10^{-16}$ \\
CE  & $H^{-1/2}$ & tanh & 1.0000000000 & $1.54\times10^{-16}$ \\
CE  & $H^{-1}$   & ReLU & 1.0000000000 & $2.54\times10^{-16}$ \\
CE  & $H^{-1}$   & tanh & 1.0000000000 & $2.43\times10^{-16}$ \\
MSE & $H^{-1/2}$ & ReLU & 1.0000000000 & $2.71\times10^{-16}$ \\
MSE & $H^{-1/2}$ & tanh & 1.0000000000 & $1.48\times10^{-16}$ \\
MSE & $H^{-1}$   & ReLU & 1.0000000000 & $1.92\times10^{-16}$ \\
MSE & $H^{-1}$   & tanh & 1.0000000000 & $1.56\times10^{-16}$ \\
\bottomrule
\end{tabular}
\end{table}

\section{Secondary Experiments}
\label{app:secondary-experiments}
The three comparisons below are referenced from \cref{sec:controls}. The chain baselines share the width, optimizer, and epoch budget of SpatialPC and differ in topology.

\subsection{Chain control}
The finite-round Moore model is compared with a 256-state chain at the same state count, width, round budget, objective, and optimizer (\cref{tab:ablation-topology}). With $T=32<256$ no input information reaches the output of the chain (\cref{prop:forward-causality}), so its prediction is the constant readout of a zero state and its task gradient vanishes; the reported accuracies are the frequencies of class $0$ in the two test sets. The chain also changes the depth-dependent residual scaling of the implementation, so the comparison is a qualitative reference rather than an isolated topology intervention.

\begin{table}[htbp]
\caption{Chain control at matched $N=256$, $T=32$, $H=128$. With $T<N$ the chain output never receives input information, so accuracy equals the class-0 frequency and the task gradient is zero.}
\label{tab:ablation-topology}
\centering
\small
\begin{tabular}{@{} lll *{2}{r@{\,{\scriptsize$\pm$}\,}>{\scriptsize}l} @{}} 
\toprule
\textbf{Dataset} & \textbf{Method} & \textbf{Graph} & \multicolumn{2}{c}{\textbf{Test Acc. (\%)} $\uparrow$} & \multicolumn{2}{c}{\textbf{Task-grad cos. to settled} $\uparrow$} \\
\midrule
MNIST         & SpatialPC-v1 & Chain &  9.80 & 0.00 & 0.0000 & 0.0000 \\
              &              & Moore & 92.75 & 0.23 & 0.9283 & 0.0086 \\
\addlinespace
              & SpatialPC-v2 & Chain &  9.80 & 0.00 & 0.0000 & 0.0000 \\
              &              & Moore & 91.22 & 0.27 & 0.6467 & 0.0279 \\
\midrule
Fashion-MNIST & SpatialPC-v1 & Chain & 10.00 & 0.00 & 0.0000 & 0.0000 \\
              &              & Moore & 82.85 & 0.18 & 0.9745 & 0.0008 \\
\addlinespace
              & SpatialPC-v2 & Chain & 10.00 & 0.00 & 0.0000 & 0.0000 \\
              &              & Moore & 80.66 & 0.18 & 0.8795 & 0.0231 \\
\bottomrule
\end{tabular}
\end{table}

\subsection{Activation ablation}
\Cref{tab:ablation-activation} compares ReLU and tanh at the reference configuration. Tanh gives slightly lower accuracy and consistently higher gradient cosine for both protocols, consistent with a smaller Jacobian norm and hence a faster path tail.

\begin{table}[htbp]
\caption{Activation ablation at the reference configuration ($N=256$, $T=32$, $H=128$, $\lambda=100$, one epoch).}
\label{tab:ablation-activation}
\centering
\small
\begin{tabular}{@{} lll *{2}{r@{\,{\scriptsize$\pm$}\,}>{\scriptsize}l} @{}} 
\toprule
\textbf{Dataset} & \textbf{Method} & \textbf{Activation} & \multicolumn{2}{c}{\textbf{Test Acc. (\%)} $\uparrow$} & \multicolumn{2}{c}{\textbf{Task-grad cos. to settled} $\uparrow$} \\
\midrule
MNIST         & SpatialPC-v1 & ReLU & 92.75 & 0.23 & 0.9283 & 0.0086 \\
              &              & Tanh & 91.63 & 0.25 & 0.9628 & 0.0078 \\
\addlinespace
              & SpatialPC-v2 & ReLU & 91.22 & 0.27 & 0.6467 & 0.0279 \\
              &              & Tanh & 88.86 & 0.19 & 0.9539 & 0.0031 \\
\midrule
Fashion-MNIST & SpatialPC-v1 & ReLU & 82.85 & 0.18 & 0.9745 & 0.0008 \\
              &              & Tanh & 82.14 & 0.13 & 0.9843 & 0.0037 \\
\addlinespace
              & SpatialPC-v2 & ReLU & 80.66 & 0.18 & 0.8795 & 0.0231 \\
              &              & Tanh & 80.02 & 0.21 & 0.9648 & 0.0072 \\
\bottomrule
\end{tabular}
\end{table}

\subsection{External width sweep}
Depth-128 chain implementations of backpropagation, standard PC, $\mu$PC, and PC-ALM are compared with SpatialPC after one epoch over $H\in\{8,16,32,64,128\}$ (\cref{tab:exp1}). The Moore graph has $\tfrac32N-2$ edges against $N-1$ for the chain, so the two differ in edge count as well as topology.

\begin{table}[htbp]
\centering
\small
\setlength{\tabcolsep}{4pt}
\caption{One-epoch test accuracy (\%) at matched hidden width $H$. SpatialPC uses the $N=256$ Moore graph; the chain baselines use depth 128. Batch sizes are 32 for SpatialPC and 64 for the baselines. Bold marks the best entry in each row.}
\label{tab:exp1}
\begin{tabular}{lrcccccc}
\toprule
\textbf{Dataset} & \textbf{Width} & \textbf{BP} & $\bm{\mu}$\textbf{PC} & \textbf{PC} & \textbf{PC-ALM} & \textbf{SPC-v1} & \textbf{SPC-v2} \\
\midrule
MNIST  &   8 & 38.80 & 18.67 &  9.53 & 37.90 & \textbf{83.00} & 65.70 \\
       &  16 & 75.85 & 35.65 &  9.51 & 71.47 & \textbf{87.60} & 77.94 \\
       &  32 & 87.78 & 61.02 &  8.98 & 85.79 & \textbf{89.78} & 84.91 \\
       &  64 & \textbf{91.51} & 74.98 &  9.75 & 89.78 & 91.42 & 89.18 \\
       & 128 & \textbf{93.40} & 84.06 &  9.87 & 92.25 & 92.74 & 91.23 \\
\midrule
Fashion-MNIST &   8 & 47.56 & 24.48 & 11.41 & 43.44 & \textbf{74.88} & 64.77 \\
       &  16 & 73.22 & 45.45 &  8.15 & 72.02 & \textbf{78.08} & 71.48 \\
       &  32 & 78.63 & 64.72 &  8.58 & 77.43 & \textbf{79.86} & 75.98 \\
       &  64 & 81.23 & 72.29 & 11.09 & 80.08 & \textbf{81.60} & 78.95 \\
       & 128 & \textbf{83.39} & 76.24 &  9.25 & 82.32 & 82.85 & 80.66 \\
\bottomrule
\end{tabular}
\end{table}

\section{Algorithms} \label{app:algorithm}

\SetKwProg{Fn}{Function}{}{}
\SetKw{KwDownTo}{down to}
\SetKwFunction{LocalRound}{LocalRound}
\SetKwFunction{LocalPullback}{LocalPullback}
\SetKwFunction{Readout}{Readout}
\SetKwFunction{Partial}{Partial}
\SetKwFunction{Settled}{Settled}

\Cref{alg:moore-dag} constructs the Moore DAG of \cref{sec:moore}, \cref{alg:spatialpc-forward-step} gives the settled and finite-round forward computations of \cref{sec:forward}, and \cref{alg:spatialpc-training} gives one APD gradient evaluation on the $T$-round space--time graph (\cref{prop:finite-spatialpc-apd}).

\begin{algorithm}[H]
\caption{Moore-grid DAG construction used by SpatialPC. The Moore cycle is cut to define a topological order. Backbone edges follow that order; the cut seam and eligible interior unit-grid edges are retained as forward-oriented shortcuts.}
\label{alg:moore-dag}
\KwIn{grid side $m=2^k$, $k\ge1$}
\KwOut{DAG $G=(V,E)$ with $N=m^2$; input/root $v_{\rm in}=0$; output $O=\{N-1\}$; skip gains $s$}
\BlankLine
$w\leftarrow\texttt{LFL+F+LFL}$\;
\For{$i\leftarrow1$ \KwTo $k-1$}{
  rewrite $w$: $L\mapsto\texttt{-RF+LFL+FR-}$, $R\mapsto\texttt{+LF-RFR-FL+}$\;
}
$(p_j)_{j=0}^{N-1}\leftarrow$ points visited by a turtle reading $w$\;
$\pi(p_{(N/4+j)\bmod N})\leftarrow j$ for $j=0,\ldots,N-1$ \tcp*{choose cut and topological order}
\BlankLine
\ForEach{unit cell $c=[x,x+1]\times[y,y+1]$}{
  $n_c\leftarrow$ number of Moore-cycle edges to the left of $c$ in its row\;
  $\mathrm{interior}(c)\leftarrow(n_c\text{ is odd})$\;
}
$E\leftarrow\{(j,j+1):j=0,\ldots,N-2\}$ \tcp*{Hamiltonian backbone}
$E\leftarrow E\cup\{(0,N-1)\}$ \tcp*{directed cut seam used by the implementation}
\ForEach{unit-grid edge $\{p,q\}$ separating two interior cells}{
  $E\leftarrow E\cup\{(\min(\pi(p),\pi(q)),\max(\pi(p),\pi(q)))\}$\;
}
$s_{uv}\leftarrow[d_{\rm in}(v)d_{\rm out}(u)]^{-1/2}$ for $(u,v)\in E$\;
\Return{$G,v_{\rm in},O,s$}\;
\end{algorithm}

\clearpage
\begin{algorithm}[H]
\caption{SpatialPC forward computation. \texttt{Partial} performs $T$ synchronous local rounds. \texttt{Settled} evaluates the same DAG exactly in topological order.}
\label{alg:spatialpc-forward-step}
\KwIn{DAG $G=(V,E)$ in topological order; input $x\in\mathbb R^{d_x}$; input node $v_{\rm in}$; outputs $O$; parameters $\theta=(\{W_{uv}\},W_{\rm in},W_{\rm out})$; gains $s_{uv},g_{uv}$; nonlinearity $\sigma$; rounds $T$; readout scale $\kappa_H$; mode $p\in\{\mathrm{settled},\mathrm{partial}\}$}
\KwOut{prediction $\hat y$ and states $h$}
\BlankLine
\eIf{$p=\mathrm{settled}$}{
  $(\hat y,h)\leftarrow$ \Settled{$x$}\;
}{
  $(\hat y,h)\leftarrow$ \Partial{$x,T$}\;
}
\Return{$\hat y,h$}\;
\BlankLine
\Fn{\Settled{$x$}}{
  $h^\star\leftarrow\mathbf0\in\R^{N\times H}$\;
  $h^\star_{v_{\rm in}}\leftarrow W_{\rm in}x/\sqrt{d_x}$\;
  \For{$v$ in topological order after $v_{\rm in}$}{
    $h^\star_v\leftarrow\sum_{u\in\pa(v)}[s_{uv}h^\star_u+g_{uv}W_{uv}\sigma(h^\star_u)]$\;
  }
  \Return{\Readout{$h^\star$},$h^\star$}\;
}
\BlankLine
\Fn{\Partial{$x,T$}}{
  $h^0\leftarrow\mathbf0\in\R^{N\times H}$\;
  \For{$t\leftarrow0$ \KwTo $T-1$}{
    $h^{t+1}\leftarrow$ \LocalRound{$h^t,x$}\;
  }
  \Return{\Readout{$h^T$},$h^T$}\;
}
\BlankLine
\Fn{\LocalRound{$h,x$}}{
  \ForPar{$(u,v)\in E$}{
    $\mu_{uv}\leftarrow s_{uv}h_u+g_{uv}W_{uv}\sigma(h_u)$\;
  }
  $h'\leftarrow\mathbf0\in\R^{N\times H}$\;
  $h'_{v_{\rm in}}\leftarrow W_{\rm in}x/\sqrt{d_x}$\;
  $h'_v\leftarrow\sum_{u\in\pa(v)}\mu_{uv}$ for $v\ne v_{\rm in}$\;
  \Return{$h'$}\;
}
\BlankLine
\Fn{\Readout{$h$}}{
  \Return{$\kappa_H W_{\rm out}|O|^{-1/2}\sum_{o\in O}h_o$}\;
}
\end{algorithm}

\begin{algorithm}[H]
\caption{SpatialPC training by Adjoint Prediction Dynamics (APD) on the $T$-round space--time graph.}
\label{alg:spatialpc-training}
\KwIn{sample $(x,y)$; $G,\theta,s,g,\sigma,T,\kappa_H$ as in \cref{alg:spatialpc-forward-step}; feasibility weight $\lambda$}
\KwOut{$\nabla_\theta\mathcal J_T$}
\BlankLine
$h^0\leftarrow\mathbf0$\;
\For{$t\leftarrow0$ \KwTo $T-1$}{
  $h^{t+1}\leftarrow$ \LocalRound{$h^t,x$}\;
}
$\hat y\leftarrow$ \Readout{$h^T$}\;
$r^T\leftarrow h^T-{}$\LocalRound{$h^T,x$}\;
\BlankLine
\tcp{terminal APD forcing on the unrolled graph}
$a^T\leftarrow\nabla_{h^T}\ell+\frac{\lambda}{NH}\big(r^T-{}$\LocalPullback{$h^T,r^T$}$\big)$\;
$\nabla W_{\rm out}\leftarrow\kappa_H\frac{\partial\ell}{\partial\hat y}\big(|O|^{-1/2}\sum_{o\in O}h_o^T\big)^\top$\;
$\nabla W_{uv}\leftarrow-\frac{\lambda}{NH}g_{uv}r_v^T\sigma(h_u^T)^\top$ for $(u,v)\in E$\;
$\nabla W_{\rm in}\leftarrow-\frac{\lambda}{NH}r^T_{v_{\rm in}}x^\top/\sqrt{d_x}$\;
\BlankLine
\tcp{APD: each node predicts its adjoint from local downstream credit}
\For{$t\leftarrow T-1$ \KwDownTo $0$}{
  \ForPar{$(u,v)\in E$}{
    $\nabla W_{uv}\leftarrow\nabla W_{uv}+g_{uv}a_v^{t+1}\sigma(h_u^t)^\top$\;
  }
  $\nabla W_{\rm in}\leftarrow\nabla W_{\rm in}+a_{v_{\rm in}}^{t+1}x^\top/\sqrt{d_x}$\;
  $a^t\leftarrow$ \LocalPullback{$h^t,a^{t+1}$}\;
}
\Return{$\nabla_\theta\mathcal J_T$}\;
\BlankLine
\Fn{\LocalPullback{$h,c$}}{
  \ForPar{$(u,v)\in E$}{
    $\nu_{vu}\leftarrow s_{uv}c_v+g_{uv}\sigma'(h_u)\odot W_{uv}^\top c_v$\;
  }
  $c'_u\leftarrow\sum_{v\in\ch(u)}\nu_{vu}$ for $u\in V$\;
  \Return{$c'$}\;
}
\end{algorithm}

\end{document}